\documentclass[10pt]{article} 

\usepackage[accepted]{tmlr}

\usepackage{amsmath,amsfonts,bm}

\def\eqref#1{equation~\ref{#1}}

\def\1{\bm{1}}

\DeclareMathAlphabet{\mathsfit}{\encodingdefault}{\sfdefault}{m}{sl}
\SetMathAlphabet{\mathsfit}{bold}{\encodingdefault}{\sfdefault}{bx}{n}

\newcommand{\R}{\mathbb{R}}

\usepackage{url}
\usepackage{wrapfig}
\usepackage{lipsum} 

\title{The Global Empirical NTK: Self-Referential Bias and Dimensionality of Gradient Descent Learning} 

\author{\name James Hazelden \email jhazelde@uw.edu\\
        \addr Department of Applied Mathematics, University of Washington\\
        The Allen Institute 
        \AND
        \name Laura Driscoll \email laura.driscoll@alleninstitute.org\\
        \addr The Allen Institute \\
        Department of Neurobiology \& Biophysics, University of Washington
        \AND
        \name Eli Shlizerman \email shlizee@uw.edu\\
        \addr Department of Applied Mathematics, University of Washington\\
        Department of Electrical \& Computer Engineering, University of Washington
        \AND
        \name Eric Shea-Brown \email etsb@uw.edu\\
        \addr Department of Applied Mathematics, University of Washington\\
        Department of Neurobiology \& Biophysics, University of Washington\\
        The Allen Institute
}

\def\month{10}  
\def\year{2026} 
\def\openreview{\url{https://openreview.net/forum?id=tPjNd8S8eb}} 

\usepackage{amsmath,amsfonts,amsthm}
\usepackage{booktabs}
\usepackage{nicefrac}
\usepackage{microtype}
\usepackage{graphicx}
\usepackage{enumitem}
\usepackage{hyperref}
\usepackage{xcolor}
\usepackage{overpic} 

\graphicspath{{figures/}}  

\newtheorem{proposition}{Proposition}
\newtheorem{theorem}{Theorem}
\newtheorem{corr}{Corollary}      

\newcommand{\err}{\text{err}}
\newcommand{\im}{\text{Im}}

\newcommand{\todo}[1]{}

\newcommand{\rank}{\text{rank}}

\newcommand{\K}{\mathcal{K}}

\renewcommand{\P}{\mathcal{P}}

\newcommand{\NTK}{\text{GeNTK}}
\newcommand{\F}{\mathcal{F}}
\newcommand{\cat}{\text{cat}}

\newcommand{\Lcomment}[1]{}
\newcommand{\EcommentDONE}[1]{}

\renewcommand{\vec}{\text{vec}}

\begin{document}

\maketitle

\begin{abstract}
    In training a neural network with gradient descent (GD), each iteration induces a linear operator that governs the first-order updates to a model's internal state variables. We define this operator as the Global Empirical Neural Tangent Kernel (\(\NTK\)). In finite-width networks, \(\NTK\) is typically intractable to form, leading prior work to focus on restrictive settings, such as tracking outputs only or taking the limit of infinitely many neurons. Here, we bridge this gap by studying the structure of \(\NTK\) for a wide range of models. Formulating the model state as the solution to a single global implicit constraint, we derive \(\NTK\) as a product of two operators: \(\K\), accounting for immediate parameter-to-state interactions, and \(\P\), describing internal state-to-state dependencies. For a broad class of weight-based models, including RNNs and transformers, we prove a universal Kronecker-core theorem showing that \(\K\) admits an exact, forward-pass-computable form given by the Gram matrix of weight-site variables. {This core gives a measurable state-dependent restriction on Euclidean GD, which we call \textit{self-referential bias}: current weight-site activity filters which adjoint directions can drive a first-order update.}  For recurrent models (GRUs and RNNs), we examine the spectrum of \(\NTK\) and show when it is biased and low-rank in space or time under the proposed decomposition. We further demonstrate that the structure of the model dynamics at initialization biases \(\NTK\), restricting learning and preventing task components from being learned effectively. {We additionally validate the factorization and weight-site construction in a self-attention transformer and in pretrained Pythia transformer language models with up to 410M parameters. Overall, our results expose tractable local rank and mode-selective biases in the finite-width GD geometry.} To further enable use of \(\NTK\) as a practical metric, we build a library, \href{https://github.com/meeree/gentk}{gentk} relying on randomized matrix-free numerical linear algebra.
\end{abstract}
\newpage

\section{Introduction}

\paragraph{Background and Motivation} Gradient descent (GD) updates model dynamics indirectly by perturbing its parameters $\theta$. If \(h(\theta)\) denotes the model variables of interest, such as the hidden state activities in a neural network, and \(J_\theta := D_\theta h(\theta)\) is the corresponding Jacobian with respect to the parameters, then a first-order parameter perturbation \(\delta \theta\) induces a change \(J_\theta \delta\theta\). Thus, the positive semidefinite operator \(J_\theta J_\theta^*\) captures the local geometry through which GD drives corrections in the state space of the variable \(h\). Specifically, it describes the first-order process of transforming task error signals into corrections to \(h\). This operator is the \textit{Neural Tangent Kernel} (NTK), describing at each iteration the local state-space geometry of GD, as well as how the parameters bias first-order training dynamics~\citep{jacot2018neural}.

Research on the NTK operator has been developed primarily along two complementary directions. One emphasizes infinite-width limits, where the NTK becomes analytically tractable and has been used to study training dynamics and generalization from a variety of perspectives~\citep[e.g.,][]{Lee_2020,baratin2021implicitregularizationneuralfeature,shan2022theoryneuraltangentkernel,canatar2022outofdistributiongeneralizationkernelregression,bordelon2025dynamicallylearningintegraterecurrent}. The other perspective studies finite-width \emph{empirical} NTK, often focused on diagnostics, fast computation and approximation of this object, with comparatively fewer works developing explicit analytical characterizations of its structure in broad contexts~\citep{novak2022fastfinitewidthneural,george_nngeometry,mohamadi2023fast,fan2020spectra,hanin2020finitedepthwidth,xiang2025neuraltangentknowledgedistillation}. 

{Our results complement infinite-width NTK theory with exact finite-width identities: the decomposition is defined at a finite state--parameter tuple and can be probed directly without assuming convergence to an asymptotic kernel.}

In this work, we formulate and study the \textit{Global-State Empirical NTK} (\(\NTK\)), obtained by collecting all internal variables of a model into a single tensor \(h \in S\), which we call the \textit{global state}. For example, in a multilayer perceptron evaluated on a batch of inputs, \(h\) contains the activations at every layer and for every input in the batch. For recurrent, implicit, or continuous-time models, \(h\) similarly collects all hidden-state trajectories across task inputs. The resulting operator \(\NTK\) therefore acts on the full state space \(S\).  This contrasts with many prior studies, which focus on operators that act only on the output space \citep{jacot2018neural}. In this sense, the usual output NTK can be viewed as a reduced view of this larger operator. This global formulation is practically useful because it makes the first-order GD dynamics closed at the level of the model state itself, allowing us to study how GD reshapes internal representations and dynamics, not just the output.

\paragraph{Theoretical Findings} In practice, \(\NTK\) is intractable to compute explicitly since it contains \(d^2\) when the state has \(d\) entries. The hidden state tensor can quickly become prohibitively large, typically scaling with batch size, as well as the layer and hidden unit count. 
To overcome this, we derive structural theorems demonstrating that the operator is both mathematically tractable and practically predictive. In particular, we show that, for a broad class of models specified by a constraint on all state variables, the \(\NTK\) admits a decomposition into a product of simpler operators, \(\P\) and \(\K\), describing state-to-state and immediate parameter-to-state gradient propagation, respectively (Proposition~\ref{prop:ntk}). In particular, the operator can be written in the form
\[
\NTK = \P \K \P^*
\]
For example, for a recurrent model, \(\K\) is an integral operator defining the immediate impact of parameters on the one-step dynamics, while \(\P\) is the causal Green's operator describing the model's linearized dynamics (Corollary~\ref{corr:recntk}). 

{Our main theoretical result shows that, after an equivalent lifting of a weight-based computation, the core has an explicit, computable form that constrains the rank and structure available to the full NTK.} Specifically, for any model where parameters enter the internal dynamics through matrix-vector products (including models with nesting or recurrence) we show that \(\K\) has a universal Kronecker structure \citep{nielsen2010quantum},
\[
\K = V V^* \otimes I
\]
where \(V\) collects the activity at the model ``weight sites'' (Theorem~\ref{thm:core}).  These sites consist of all variables that the model weights multiply. For an RNN, this includes the task input and all hidden activations, while for a self-attention + MLP transformer, these weight sites consist of the task input and attention output (Appendix~\ref{subsec:corederivap}). Importantly, the weight sites are often both readily computable and of high interest in efforts to interpret model computation and dynamics. 

{The exact structural statement is that the lifted core bounds the algebraic rank available to the full operator, while propagation through \(\P\) can strongly reshape its spectrum and effective rank. We separately examine empirically when the dominant spectrum of the full \(\NTK\) resembles that of the core.} In total, \(\NTK\) is decomposable into two interpretable operators: (1) the core \(V V^*\) which is architecture dependent and easy to compute, and (2) the operator \(\P\), encoding state-space dependencies of the model. Each operator constrains the spectral structure of the \(\NTK\), manifesting as implicit bias of GD toward particular tasks or low-rank dynamical solutions.

Due to few simplifying assumptions on the model, our results concern the landscape of GD at a given iteration.  In particular, the GeNTK describes the parameter-to-state geometry through which GD will produce corrections at each iteration, but this operator and the task losses evolve. However, if it is highly structured or low-rank, it does imply that GD will be biased, at least to first order near a particular snapshot. Actually integrating this process, i.e., closing the loop, for a particular learning rule such as traditional GD or ADAM is out of the scope of this work. All experiments in this work use discrete learning, not the continuous gradient flow. 

Our construction uses an implicit, equivalent reparameterization of the explicit model. The choice of constraint, \(\F\) and the resulting operators \(\P\) and \(\K\) are not unique, but the $\NTK$, defined by the Gram matrix of the empirical sensitivity, $D_\theta h$ is. Thus, the decomposition should be viewed as a family of representations of the same full operator: factor-specific interpretations and the tightness of bounds may depend on the chosen representation. This non-uniqueness means that one can reparameterize the model in the form of other equivalent ways depending on the desired assumptions (e.g., the construction in the proof of Theorem~\ref{thm:core}). The choice of \(\F\) and \(V\) are certainly not unique: one could reparameterize certain models in a way that a more substantial bottle-neck emerges, so our result yields a specific upper bound on the GeNTK rank: one which can be computed easily for many models and explains why GD produces structured, low-rank corrections.

{Corollary~\ref{corr:adjointalign} uses the structure \(\NTK = \P (V V^T \otimes I_n) \P^*\) into a product of operators with the core \(V V^T\) to quantify what we term \textit{self-referential bias}. Specifically, because GD is filtered through the core $VV^*$, only adjoint components overlapping the span of the current weight-sites can drive learning.}
This allows us to quantify an important tradeoff: models producing low-rank dynamics, although often interpretable, can make learning highly restrictive by limiting which tasks are accessible to GD from a given initialization. 

Proposition~\ref{prop:spacetime} explores how low-rank structure in \(\NTK\) constrains learning. We define reduced views of the NTK that capture its behavior over hidden units (space) and over batches and timesteps (time). Crucially, if either of these reduced operators is low-rank, then GD is biased to produce low-rank state-space corrections. Thus, for GD to produce expressive updates, \(\NTK\) must be sufficiently rich in both its temporal structure (e.g., capturing both low- and high-frequency features) and its spatial structure, both of which we measure explicitly in Section~\ref{subsec:regimes}.

\begin{itemize}[leftmargin=1.8em,itemsep=0.1em,topsep=0.3em]
    \item Proposition~\ref{prop:ntk} derives the NTK decomposition
    \(
    \NTK = \P \K \P^*,
    \)
    \item Corollary~\ref{corr:recntk} characterizes the exact structure of \(\P\) and \(\K\) for recurrent models,
    \item Theorem~\ref{thm:core} demonstrates the exact, computable form
    \(
    \K = V V^* \otimes I_n,
    \)
    for any weight-based model,
    \item {Corollary~\ref{corr:adjointalign} formalizes a state-dependent restriction of GD that we call self-referential bias,}
    \item Proposition~\ref{prop:spacetime} provides minimal, testable conditions for low-rank learning under Theorem~\ref{thm:core}.
\end{itemize}

\paragraph{Experiments and Library} To computationally illustrate and study the structure of \(\NTK\), we develop a library, \href{https://github.com/meeree/gentk}{gentk}, relying on randomized matrix-free numerical linear algebra (Appendix~\ref{sec:aprnla}). These approaches lead to speedups of multiple orders of magnitude over constructing the full empirical NTK and doing the same analysis. Our package maps one-to-one to the theory, supporting tensor products, operator composition, partial traced and randomized trace estimation, so that statements such as Theorem~\ref{thm:core} can be easily verified, with syntax as in Appendix~\ref{eqn:verifythm}.

We use the theory developed here and \texttt{gentk} library to study low-rank structure in \(\NTK\), as well as self-referentially biased learning in recurrent models (GRU and RNN) and non-recurrent models (MLP + self-attention transformer). In these mechanistic examples, we find that \(\NTK\) is highly structured, with low-rank structure closely related to the corresponding weight-site activity \(V\). For GRU, we show that initial hidden-state dynamics bias the NTK (Section~\ref{subsec:grucore}), so that target components poorly represented in the current temporal modes can be learned only very slowly under Euclidean GD (Section~\ref{subsec:grubias}). Furthermore, in a student-teacher setting (Section~\ref{subsec:regimes}), we show how input dimension and model initialization define distinct low-rank spatial and temporal regimes of \(\NTK\) for a recurrent model. {A width sweep of this RNN setup, with both recurrent and input weights trainable, shows that temporal effective rank saturates while spatial effective rank continues to grow (Section~\ref{subsec:widthsweep}).} Finally, to demonstrate broader applicability, we apply our methods to transformer architectures. We show how the input representation controls a temporal bottleneck in the self-attention example (Section~\ref{subsec:transformer}) and {independently validate the factorization and weight-site core on pretrained Pythia language models ranging from 70M to 410M parameters (Appendix~\ref{subsec:pythia}).}


\section{Related Work}

\paragraph{NTK and Linearized Views of Learning}
A central line of work studies gradient descent through the Neural Tangent Kernel (NTK), which becomes effectively fixed in the infinite-width limit and yields a linearized description of learning~\citep{jacot2018neural, allen2019convergence, mei2018mean, bordelon2020spectrum}. The infinite-width perspective has been extended beyond feedforward networks, including to vanilla RNNs, neural ODEs, and transformers~\citep{alemohammad2020recurrent,yang2020tensorprogramsII,feng2023deqntk}. Related work on lazy training, feature learning, finite-width kernel dynamics, and kernel regimes further clarifies when learning remains close to this initialization-dependent linearization and how it departs from it in richer regimes~\citep{chizat2019lazy, pmlr-v139-yang21c, bordelon2023dynamics, bordelon2024feature}. In contrast, here we propose to study the finite-width \emph{empirical} NTK acting on the \emph{global state} of the model, i.e. all variables involved in evaluation, rather than an infinite-width output kernel derived for a particular architecture. Our framework yields an explicit factorization of this operator, giving predictions about learning bias and rank bottlenecks related to the formation of the hidden state latent dynamics. Finally, our implicit global formulation of the dynamics (Equation~\ref{eqn:implicit}) is reminiscent of deep equilibrium models~\citep{bai2019deep} (DEQ). But while DEQ uses this form to study equilibrium or effectively infinite-depth models and compute exact gradients, we use it to provide a convenient reformulation of a general explicit model, from which the structural properties of \(\NTK\) follow more naturally.

\paragraph{Learning Dynamics in Recurrent and Dynamical Models}
{A complementary line of theory examines learning in recurrent or continuous-depth models through dynamical-systems and optimization viewpoints.} Classical work connected backpropagation to adjoint-state methods and optimal control \citep{lecun1988theoretical, pearlmutter1995gradient, pontryagin1962mathematical}. Later work extended these approaches to Neural ODEs and related continuous-depth models \citep{weinan2017proposal, chen2018neural}. More recently, exact learning dynamics have been derived in special linear recurrent settings, yielding sharp descriptions of how gradient descent builds long timescales and low-dimensional structure \citep{bordelon2025dynamicallylearningintegraterecurrent}. {Our work is similar in spirit to such dynamical viewpoints, however it targets nonlinear, finite-width models where closed-form training trajectories are generally unavailable. Instead of solving a continuous training trajectory explicitly, we derive a general operator decomposition that makes \(\NTK\) structurally interpretable across models.} 

\paragraph{Dynamical Analyses of Learned Recurrent Representations}
A large empirical and theoretical literature studies trained recurrent networks through their learned dynamical structure, including fixed points, attractors, low-dimensional manifolds, and dynamical motifs \citep{sussillo2013opening, yang2019task, Farrell2022, turner2023simplicity, driscoll2024flexible,RecanatesiFarrell2021, NEURIPS2020sch}. Related perspectives in simplicity bias and low-rank learning likewise emphasize that GD often favors structured, low-dimensional, or shared solutions~\citep{turner2023simplicity,pellegrino2023low,farrell2023from,Myrov2026Hierarchical}.  The study \citep{pellegrino2023low} is particularly connected to our work, as it derives bounds for the rank of parameter updates in terms of adjoints and state-space corrections. In contrast to this work, our primary concern here is with state-space correction rank and, more broadly, the spectral structure and bias of state-space corrections, utilizing \(\NTK\). 

In recurrent models, related work also suggests that learning and memory can improve near the edge of chaos, where dynamics remain expressive without becoming strongly unstable~\citep{rajan2010stimulusdependent,mastrovito2024transition}. Our work is complementary to this literature. Rather than analyzing learned dynamics only after training, we study the operator governing \emph{state corrections during training}. {The resulting state-dependent restriction, which we call \emph{self-referential bias}, provides one explicit mechanism connecting current model/input activity to the directions most accessible to a GD update. We view this as complementary to, rather than a general explanation of, prior observations of simplicity bias and low-dimensional learned dynamics.}

\paragraph{Implicit Regularization and Implicit Bias}
{A related line of work studies the implicit regularization or implicit bias of gradient methods, i.e., how optimization can select particular solutions even without an explicit regularizer. For example, gradient descent on separable classification problems converges toward a max-margin direction~\citep{soudry2018implicit}, while in matrix factorization it can favor low-rank or otherwise low-complexity solutions, with depth further shaping this bias~\citep{gunasekar2017implicit,arora2019implicit}. Related results extend this perspective to tensor factorization and hierarchical tensor factorizations corresponding to deep convolutional networks, where optimization can favor low tensor-rank or low hierarchical-rank structure~\citep{razin2021implicit,razin2022hierarchical}. Our use of \emph{implicit bias} is related but distinct: these works primarily characterize the solution selected over an optimization trajectory in particular models, whereas we study the iteration-local state-space geometry of a finite-width model. In our setting, \(\NTK\) determines which state-space corrections are amplified or suppressed at a fixed state--parameter tuple. Thus, the two viewpoints concern different levels of the learning dynamics and are complementary.}

\paragraph{Natural Gradient, Fisher Information, and Kronecker Structure}
Second-order and natural-gradient methods based on structured curvature approximations, such as K-FAC \citep{martens2015optimizing, grosse2016kronecker, martens2020new}, are somewhat related to our work in that they also exploit Kronecker structure in objects tied to the Fisher information or NTK. K-FAC, for example, assumes a layerwise Kronecker factorization of the Fisher information to approximate parameter-space curvature and efficiently precondition gradient descent. In contrast, we study the finite-width empirical NTK on the global state space and show that certain operators possess exact Kronecker structure. Thus, the connection is conceptual rather than algorithmic: in our setting, Kronecker structure is derived exactly, not imposed as an approximation.

{Our factorization also provides a direct way to describe how parameter-space preconditioning changes the induced state-space update; we make this connection explicit in the Discussion.}

\section{The Global-State NTK Decomposes into P and K}

\subsection{Any Model as an Implicit Constraint}

\paragraph{The Global State}

Any parameterized, input-driven model fundamentally consists of a system of variables, constrained by the parameters and task inputs. Throughout, we use the following notation.
\begin{align}
\begin{split}
    h \in S &\quad \text{The \textit{Global State Space}}, \\
    x \in X &\quad \text{The \textit{Task Input Space}}, \\
    \theta \in P &\quad \text{The \textit{Parameter Space}}.
\end{split}
\end{align}
We let \(S\) and \(X\) be generic, reflecting the diversity of models encountered in deep learning and control. 
Here, global signifies that \(h\) and \(x\) consist of \textit{all} simultaneous variables involved in the model evaluation, rather than a specific snapshot such as the model output or activations at a particular layer. 
Tracking the full global state makes learning dynamics closed at the state level, such that the task-dependent error signal acting is the only external driving term in GD. 

\paragraph{Any Model as a System of Constraints} The state, task inputs and parameters are described by a single model-defining constraint \(\mathcal{F} : (S, X, P) \rightarrow S\). The space of admissible states in \(S\) is implicitly specified by  
\begin{align}
    \label{eqn:implicit} 
    \mathcal{F}(h, x, \theta) = 0.
\end{align}
Architectural assumptions impose corresponding algebraic structure on the constraint \(\F\). {This constraint representation is not unique: equivalent choices of internal variables, ordering, or lifted constraints may yield different factors below while describing the same fixed model map.} For example, for recurrent architectures, the state Jacobian \(D_h \F\) is lower-triangular in time, reflecting causal propagation, as in the following illustration. 

\subsubsection{Discrete Recurrent Example}
\label{subsec:recglobal}
Briefly, we illustrate the implicit reformulation for a generic recurrent model, to which we return throughout in order to make each additional structural result concrete. In the discrete case, let \(x \in \R^{n_x \times n_t \times n_{in}}\) denote a 3-tensor of \(n_x\) distinct task inputs. Here, each distinct task input \(x_j\) is a trajectory in \(\R^{n_{in}}\) evaluated at \(n_t\) time points. Assume a recurrent model driven by these inputs and has sequential dynamics given by
\begin{align}
    h_j(t+1) = &f(h_j(t), \,x_j(t+1), \,\theta), \\
    h_j(0) := 0,  &\text{ for } j=1,\dots,n_x; \, t = 1,\dots,n_t-1.
\end{align}
Then, this can be reformulated as a global-state constraint. Specifically, let \(X = \R^{n_x \times n_t \times n_{in}}, S = \R^{n_x \times n_t \times n_h}\) denote the input and state spaces, respectively. So, \(h \in S\) is the stacked 3-tensor of all task-conditioned hidden state evaluations. Finally, the dynamics can be equivalently formulated as 
\begin{align}
    \label{eqn:frec}
    \mathcal{F}_{rec}(h, x, \theta) := h - f( T_\downarrow h, x, \theta) =  0.
\end{align}
where \(T_{\downarrow} : S \rightarrow S \) is a linear transformation shifting times backwards by one timestep, \(T_{\downarrow} = I_{n_x} \otimes U \otimes I_{n_h}\) where \(U\) is an \(n_t\) by \(n_t\) matrix with ones on the lower diagonal.

\paragraph{Continuous Case} For a continuous time dynamical system with initial condition \(0\), evaluated at times \(t \in [0,\tau]\), the constraint takes the form \(D_t h - f(h, x, \theta) = 0\), where \(D_t : S \rightarrow S\) is the time-derivative operator. This defines the tangential dynamics of the model along each input-conditioned trajectory. In this case, \(S\) is an infinite-dimensional space consisting of bundles of input-driven trajectories in \(\R^{n_h}\), but the same structural results naturally apply in this continuous setting.

{
\paragraph{Construction for explicit acyclic models}
For an explicit model with an acyclic computation graph, collect the internal graph variables into \(h\) and write the computation as \(\F(h,x,\theta)=h-G(h,x,\theta)=0\). Choose any topological ordering and write one residual constraint \(F_i=h_i-G_i\) for each node. This partitions graph dependencies into state-to-state edges, which determine \(A:=D_hG\), and parameter-to-state edges, which determine \(D_\theta\F\); for weight-based edges, the variables immediately multiplied by the weights are stacked to form the weight-sites \(V\) in Theorem~\ref{thm:core}. The choice of topological ordering and equivalent lifting is not unique.

In a topological ordering, \(A\) is strictly block triangular and nilpotent (corresponding to dependencies eventually terminating at a certain depth). Hence the propagation operator has the finite path expansion
\[
\P=(D_h\F)^{-1}=(I-A)^{-1}=I+A+\cdots+A^L,
\]
where \(L\) is no larger than the longest computational path. Each power collects products of local Jacobians along paths through the graph, making the relation to backpropagation explicit. Thus \(\P\) and \(\P^*\) can be applied matrix-free using Jacobian-vector and vector-Jacobian products. For cyclic or implicit models this finite expansion need not apply; when \(D_h\F\) remains invertible, their action can instead be obtained with matrix-free linear solves, for example Krylov methods. Implementation details are given in Appendix~\ref{sec:aprnla}. This expansion of \(\P\) for acyclic computations directly mirrors recent, emerging work arguing that backpropagation should be viewed as solving a nilpotent linear system~\citep{edelman2023backpropagation, boughammoura2026backpropagation}.}

\subsection{Implicit Formulation Yields a Natural Factorization}

In this section, we define the global-state NTK (\(\NTK\)), which is an operator modeling how error signals are transformed into state corrections under gradient descent (GD). 

A model is trained by choosing a loss to minimize, \(L : S \rightarrow \R\), e.g., comparing a readout extracted from the global state to a set of desired outputs in a supervised case. Assuming \(D_h \F\) is locally invertible (i.e., the dynamics are well-posed), we define two linear operators, \(\P, \K : S \rightarrow S\) and a global error
\begin{align}
    \label{eqn:pkdef} 
    \P := (D_h \F)^{-1}, \quad \K := (D_\theta \F)(D_\theta \F)^*, \quad \text{err}(h) := \nabla_h L. 
\end{align}
Where \(*\) denotes the Hermitian adjoint (conjugate transpose in the finite case). The error, \(\err(h) \in S\) encodes immediate loss gradients. For example, if \(L(h) = \frac{1}{2} \mathbb{E}_h [\|h - h^*\|^2]\) then \(\err(h) = h - h^*\).
 
\begin{figure}[t]
    \centering
    \includegraphics[width=.99\linewidth]{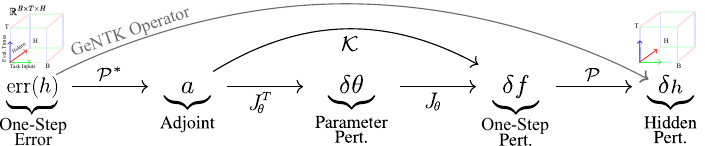}
    \caption{\textbf{Backpropagation of errors for a recurrent model, annotated by the corresponding operators, \(\P, \K\) and \(\NTK\) from Proposition~\ref{prop:ntk}}. The operator \(\P^*\) maps the global error signal to state adjoint sensitivity, describing how the loss depends on the state, \(h\). Then, these adjoints are projected into and out of the parameter space by \(\K\), potentially zeroing or misdirecting them. Finally, the corrections, which modify the one-step model dynamics, are accumulated over time to produce a correction, \(\delta f\), by forward propagating through the operator \(\P\). We show that, to first-order (small learning rate), backpropagation-through-time and the adjoint method \citep{rumelhart1986learning,pontryagin1962mathematical,chen2018neural} correspond to applying a sequence of tensor-valued operators acting on the global-state space. }
    \label{fig:schem}
\end{figure}

Let \((h, \theta)\) be a particular state-parameter pair satisfying Equation~\ref{eqn:implicit}, for which GD produces a perturbation to the model parameters in the direction of steepest descent in the parameter space (choosing learning rate \(1\) for simplicity),
\begin{align}
    \delta \theta := -\nabla_\theta L(h).
\end{align} 
According to the chain rule and implicit function theorem, we can then compute {\(\delta \theta = ((D_\theta \mathcal{F})^* \cdot (D_h \F)^{-*})(\err(h))\)}, and the first order correction to \(h\) is {\(\delta h = -((D_h \F)^{-1} \cdot D_\theta \F)(\delta \theta)\)}. Combining these and plugging in Equation~\ref{eqn:pkdef} we obtain
$$ \delta h =  -(\P \cdot \K \cdot \P^*)(\err(h)). $$
Hence, \(\NTK = \P \K \P^*\) is the specific linear operator on the (typically tensor-valued) domain \(S\) mapping \(\err(h)\) to \(\delta h\). This is summarized in the following Proposition~\ref{prop:ntk} and Figure~\ref{fig:schem}.

\setlength{\fboxsep}{.02\textwidth}   

\noindent\fbox{%
    \parbox{.96\textwidth}{%
\begin{proposition}[$\NTK$ Definition]
    \label{prop:ntk}
    Consider a model in the global implicitly constrained form above,
    $$\F(h, x, \theta) = 0.$$
    {Assume \(D_h\F\) is invertible at the state-parameter-input tuple of interest, so that the implicit function theorem locally determines \(h\) as a function of \((x,\theta)\). To first order,} each GD step updating \(\theta\) produces a perturbation to the state, \(\delta h\), defined by 
    $$\delta h = -\NTK(\err(h)), \text{ With } \, \NTK := \P \K \P^*.$$
    Where \(\P, \K  : S \rightarrow S\) are the Propagation and Parameter linear operators, respectively, and \(\err(h) \in S\) is a global error signal, as defined in Equation~\ref{eqn:pkdef}.
\end{proposition}
}}

\begin{proof}
    The proposition follows by the derivation above. In particular, from the global implicit formulation of the model, the factorization emerges naturally from the implicit function theorem and the chain rule. For specific models, however, such as the recurrent example below (Corollary~\ref{corr:recntk}), the concrete forms of \(\P\) and \(\K\) as integral operators can still be nontrivial to derive.
\end{proof}

The decomposition separates three roles: \(\K\) aggregates how parameters immediately act on the global state, \(\P\) propagates perturbations through the constrained dynamics, and \(\P^*\) propagates error signals ``backward,'' as made more explicit in the recurrent example below (Corollary~\ref{corr:recntk}) and Figure~\ref{fig:schem}. {In the finite-dimensional setting, under the invertibility assumption of Proposition~\ref{prop:ntk}, \(\P\) is algebraically full-rank and therefore \(\operatorname{rank}(\P\K\P^*)=\operatorname{rank}(\K)\). Nevertheless, ill-conditioned or anisotropic propagation can strongly concentrate the spectrum and reduce the effective rank of the full operator or of its reduced space-time views. In the lifted construction of Theorem~\ref{thm:core}, projection back to the original state space additionally yields an explicit upper bound on algebraic rank from the core.} {As \(D_h\F\) approaches singularity, \(\P=(D_h\F)^{-1}\) remains defined while invertible but can become arbitrarily ill-conditioned; at singularity, the inverse-based decomposition is no longer defined as written.} This, in turn, produces an implicit bias towards a restricted space of corrections. In Theorem~\ref{thm:core}, we explicitly construct the weight-based core and elucidate how it constrains the spectrum of the full NTK. 

\subsubsection{For Recurrent Models P is the Causal Propagator}

We briefly revisit the discrete recurrent example, specializing Proposition~\ref{prop:ntk} to the model in Equation~\ref{eqn:frec}. This recovers the recurrent operator decomposition developed in earlier work~\citep{hazelden2025kpflowoperatorperspectivedynamic} as a special case of the broader implicit formulation introduced here. In this setting, the propagation operator \(\P\) takes the form of a causal Green's operator, integrating one-step corrections over time, while \(\K\) is induced by the one-step parameter Jacobians.

\setlength{\fboxsep}{.02\textwidth}   

\noindent\fbox{%
    \parbox{.96\textwidth}{%
\begin{corr}[Recurrent \(\NTK\)]
\label{corr:recntk}
Consider the discrete recurrent model 
\begin{align*}
    \F_{rec}(h, x, \theta) &= h - f( T_{\downarrow} h, x, \theta), \\
     \text{ Where } \, h \in S := \R^{n_{x} \times n_t \times {n_h}}&, x \in X := \R^{n_x \times n_t \times n_{in}}, \theta \in P.
\end{align*}
Then
\[
\P = (I - D_h f \cdot T_{\downarrow})^{-1},
\qquad
\K = (D_\theta f)\cdot(D_\theta f)^*.
\]
Specifically, \(\P\) explicitly acts on a collection of tangential perturbations, \(\delta f \in S\) by 
\[
(\P q)_j(t) = \sum_{s\le t}\Phi_j(t,s) \cdot \delta f_j(s).
\]
Here, \(\Phi_j(t,s) = D_{h_j(s)} \, h_j(t)\) is the state-transition, describing linear propagation from time \(s\) to \(t\), on a fixed task input \(j\). Expanding \(\NTK = \P \K \P^*\) thus yields exact first-order corrections, 
\begin{align*}
\delta h_j(t) = -\sum_{t_0=1}^T \underbrace{\Phi_j(t, t_0)}_{{\color[HTML]{323031} \substack{\text{Forward} \\ \text{Propagator}}}} 
\sum_{j_1=1}^{n_x} \Big[ \sum_{t_1=1}^T \underbrace{D_\theta f(h_j(t_0)) D_\theta f(h_{j_1}(t_1))^*}_{{\color[HTML]{1B9AAA} \text{Parameter Kernel}}} \sum_{t_2=t_1}^T \underbrace{\Phi_{j_1}(t_2, t_1)^T}_{{\color[HTML]{323031} \substack{\text{Backward} \\ \text{Propagator}}}} \underbrace{\err(h_{j_1}(t_2))}_{{\color[HTML]{DB2955} \text{Error}}} \Big].
\end{align*}
\end{corr}
}}

\paragraph{Synopsis} For recurrent models, \(\P\) is lower-triangular in time, integrating state-space corrections forward through the model dynamics (the causal Green's function). \(\K\) captures how parameterization constrains the range of perturbations to the one-step dynamics, \(f\). Finally, \(\P^*\) maps the global error signal to an adjoint sensitivity by backpropagation through the Jacobian transpose \((D_h f)^*\). The operator \(\K\) then converts this sensitivity into parameter-induced corrections to the dynamics. Finally, \(\P\) propagates these corrections forward through the model, yielding the effective state correction \(\delta h\) (Figure~\ref{fig:schem}). Taken together, these operators recover the familiar structure of backpropagation through time and adjoint-based sensitivity propagation in discrete and continuous models, respectively \citep{werbos1990bptt, chen2018neural}.

\section{The Core K is a Computable Gram Matrix for Weight-Based Models}

\label{sec:core}

We now construct $\K$ in Proposition~\ref{prop:ntk} explicitly for weight-based models. Recall that \(\NTK\) is defined as 
\begin{align}
    \NTK = \P \K \P^*,
    \qquad
    \P = (D_h \F)^{-1},
    \qquad
    \K = (D_\theta \F)(D_\theta \F)^* .
\end{align}
Notably, a model may be defined by complex, nonlinear state-to-state dependencies, but still have simple immediate parameter-to-state interactions. This observation motivates the theorem below. Specifically, we show that for any weight-based model, $\K$ is a Gram matrix of forward-computable quantities.

\vspace{.5em}

\setlength{\fboxsep}{.02\textwidth}   

\noindent\fbox{%
    \parbox{.96\textwidth}{%
\begin{theorem}[Universal Kronecker core]
\label{thm:core}
Let a model be defined implicitly by \(\F(h,x,\theta)=0\), where $h \in S$ is the global state. Assume the model is weight-based, in the sense of Appendix~\ref{def:weighttied}, and that the associated lifted constraint system is locally well-posed, {meaning that its Jacobian with respect to the lifted state variables is invertible at the point of interest.} Let \(V \in \R^{k\times m}\) be the corresponding matrix of weight-sites. Then \(\NTK\) factors as  
\begin{align}
    \NTK
    \cong
    \P_{\mathrm{core}}
    \bigl(VV^* \otimes I_n\bigr)
    \P_{\mathrm{core}}^*,
\end{align}
where \(\P_{\mathrm{core}}\) is the induced propagator on \(S\), and \(\cong\) denotes equality up to re-ordering of the relevant axes.
{As a direct consequence, we obtain the following upper bound on the full rank,
\[
\operatorname{rank}(\NTK) \leq \operatorname{rank}(VV^*\otimes I_n)=n\,\operatorname{rank}(V).
\]
}
\end{theorem}
}}

\begin{proof}
Here, we intuitively outline the three steps of the proof, leaving the full details to the Appendix~\ref{proof:core}. {The procedure consists of (1) making the parameter structure explicit by \textit{lifting} the constraint to a higher-dimensional state space, (2) computing the NTK in this space, and (3) projecting back to the original state space \(S\).}

\textbf{(1) Lifted Constraint} 
The weight-based assumption corresponds to a larger system of constraints where the parameter dependence appears only through explicit matrix-vector products on intermediate weight-site variables. Specifically, we introduce a larger global state space, \(S_{aug} = (S, \R^{k \times m}, \R^{k \times n})\) and the augmented state
\begin{align}
    h_{\mathrm{aug}}=(h,V,Q),
\end{align}
where $V \in \R^{k \times m}$ are the weight sites and \(Q = V W^\top\) is the image of \(V\) under \(W\). Here \(k\) is the number of instances where the weights enter \(\F\). This yields an equivalent constraint system, 
\begin{align}
    \F_{\mathrm{aug}}(h_{\mathrm{aug}},x,W)=0.
\end{align}
whose restriction to the original $h$-coordinates reproduces the original model. This step is necessary to incorporate the simplifying assumption that any weights appear inside the computation as matrix-vector products. 

\textbf{(2) Applying Proposition~\ref{prop:ntk} to the Lifted System} 
Using this lifted system of constraints, we can apply Proposition~\ref{prop:ntk} to compute \(\NTK\) for the lifted state space, \(S_{aug}\). Since the larger system of constraints only reflects the parameter dependence through the matrix-vector products \(Q - V W^\top = 0\), the corresponding parameter operator can be written as a 3-by-3 block matrix (in the lifted state space) as 
\begin{align}
    \K_{\mathrm{aug}}
    =
    (D_W \F_{\mathrm{aug}})(D_W \F_{\mathrm{aug}})^*
    =
    \begin{pmatrix}
        0 & 0 & 0 \\
        0 & 0 & 0 \\
        0 & 0 & VV^* \otimes I_n
    \end{pmatrix}.
\end{align}
Next, the operator \(\P_{aug} = (D_{h_{aug}} \F_{aug})^{-1}\) is resolved by inverse of a 3-by-3 block matrix.

\textbf{(3) Project to the Original State}
The lifted NTK acts on \((S, \R^{k \times m}, \R^{k \times n})\), so the upper left block matrix corresponds to the original \(\NTK\) operator acting on \(S\). Computing this upper block by restriction results in
\begin{align}
    \NTK
    =
    \P_{\mathrm{core}}
    (VV^* \otimes I_n)
    \P_{\mathrm{core}}^*.
\end{align}
where \(\P_{\mathrm{core}}\) can be resolved exactly; full details are given in Equation~\ref{eqn:newp} in the Appendix.
\end{proof}

\paragraph{Interpretation}
For simpler notation, throughout the rest of the paper we drop the \({\mathrm{core}}\) prefix. Theorem~\ref{thm:core} gives an equivalent decomposition of \(\NTK\) in which the explicit \emph{weight-based} parameter structure makes \(\K\) concrete and computable. In particular, the theorem applies to models whose trainable parameters, after lifting if necessary, enter through products of the form \(Wv\) for a shared weight matrix \(W \in \R^{n \times m}\); we refer to the corresponding internal vectors \(v\) as the \emph{weight-sites}. Thus, \(\P\) is not a new propagator unrelated to Proposition~\ref{prop:ntk}, but the effective state-to-state propagator induced on \(S\) by this lifted decomposition. This applies across a broad class of architectures, including nonlinear, recurrent, gated, implicit, and attention-based models, since the theorem depends only on how the weights enter the computation. A formal definition is given in Appendix~\ref{def:weighttied}, and concrete instances for the RNN, GRU, and transformer are given in Appendix~\ref{subsec:corederivap}. The proof procedure--lifting under simplifying assumptions, using Schur complements to compute \(\P\) as an inverse, and then projecting back--also provides a general strategy for imposing additional restrictions on the model structure.

The factor \(VV^*\) records covariances of activity across the weight-sites, and is therefore directly computable from forward-pass quantities. Moreover, it is readily interpretable: as made explicit in the example below, its principal components correspond to dominant temporal modes across time and task inputs. The spatial Gram matrix \(V^*V\) is also closely related to quantities that appear in many prior studies, such as hidden-state or input covariance, and has the same effective rank as \(VV^*\). Thus, systems with low-dimensional structure in their inputs and activity--as, for example, widely observed in neural systems \citep{Recanatesi2022ScaleDependentDimensionality, Cunningham2014DimensionalityReduction, farrell2023from, Morales2023IntrinsicDynamics}--will display bottlenecks in the rank of the full NTK, as made explicit in Proposition~\ref{prop:spacetime}. For many concrete models, including those treated in Appendix~\ref{subsec:corederivap}, the weight-sites are straightforward to identify, so the factorization can often be applied directly without explicitly repeating the general lifting and projection argument.

We note again that \citep{pellegrino2023low} isolate quantities related to those in Theorem~\ref{thm:core}, and use them to derive elegant bounds on the spectrum or rank of RNN weight updates and, in some cases, RNN state updates, closely related to the bottlenecking ideas explored here. Specifically, \citep{pellegrino2023low} show how the singular values of updates are bounded by singular values of the adjoint, related to our \(\P\) operator, and state covariances, related to our \(\K\). Theorem~\ref{thm:core} enables analogous concepts to apply to a broader set of models and, as we will show below, invites a separate treatment of temporal and spatial rank concepts with many applications.

The importance of Theorem~\ref{thm:core} for the remainder of the paper is that it standardizes \(\NTK\) into two interacting objects: a model-dependent propagator \(\P\) and a universal weight-site Gram matrix \(VV^*\). This theorem underlies the findings below. {In particular, it lets us distinguish temporal and spatial bottlenecks and derive rank constraints on GD updates; the following corollary quantifies the associated state-dependent filtering by the weight-sites.}

{\paragraph{Non-Uniqueness} This lifting is a convenient construction rather than a canonical decomposition. Equivalent choices of constraint, topological ordering, or weight-site organization can produce different factor representations while leaving the fixed-state operator \(D_\theta h(D_\theta h)^*\) unchanged. Accordingly, the factors \(\P_{\mathrm{core}}\) and \(VV^*\otimes I_n\), as well as the tightness of the resulting rank bound, are representation-dependent; the theorem asserts a valid computable construction, not an optimal or unique one.}

\subsection{Example: Recurrent Model with Input and Hidden Weights}
\label{subsec:core_rnn_example}

 For the discrete recurrent example, consider the discrete-time system \(h_{t+1} = f(h_t,\; W_{\mathrm{rec}} h_t,\; W_{\mathrm{in}} x_{t+1})\), with parameters \(\theta = \mathrm{cat}\!\left(\vec(W_{\mathrm{rec}}),\vec(W_{\mathrm{in}})\right)\), where \(\cat\) forms the direct sum concatenation. Collecting the weights into the block matrix \(W = \mathrm{blockdiag}(W_{\mathrm{rec}},W_{\mathrm{in}})\), the model depends on \(W\) only through the concatenated state-input vectors \(v_t = \mathrm{cat}(h_t,x_{t+1})\) (Figure~\ref{fig:wrapped}). These are precisely the weight-sites, so the corresponding matrix is
\[
V = \mathrm{cat}(H,X)
    \in \R^{n_x n_t \times (n_{\mathrm{h}} + n_{\mathrm{in}})}.
\]
Thus, in the notation of Theorem~\ref{thm:core}, we have \(k = n_x n_t\), \(m = n_{\mathrm{h}} + n_{\mathrm{in}}\), and \(n = 2n_{\mathrm{h}}\). Theorem~\ref{thm:core} therefore yields
\begin{align}
    \NTK
    =
    \P
    \bigl(VV^* \otimes I_{2n_{\mathrm{h}}}\bigr)
    \P^*.
\end{align}

\begin{wrapfigure}{r}{0.35\textwidth}
\vspace{-1.3em}
\centering
\includegraphics[width=0.31\textwidth]{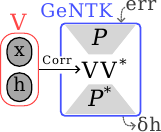}
\caption{Schematic of Theorem~\ref{thm:core} for the Example~\ref{subsec:core_rnn_example}, with \(V = \text{cat}(H,X)\), as is the case for the GRU and RNN. If the hidden units and inputs have low-dimensional activity, then this joint state matrix bottlenecks the full NTK.}
\label{fig:wrapped}
\vspace{-3.0em}
\end{wrapfigure}

Here, the Gram matrix has explicit structure, composed of all contracted inner products of the hidden and input states. Specifically, we can write 
$$(V V^*)_{j, t; j', t'} = \langle h_j(t), h_{j'}(t')\rangle + \langle x_j(t), x_{j'}(t')\rangle$$
Explicitly, it is an uncentered covariance matrix corresponding to the joint hidden-input activity.
Thus, the matrix \(V V^*\) acts like an (unnormalized) projector onto the dominant modes of the joint hidden-input activity over all timesteps and batches. Low-rank structure in \(\mathrm{cat}(h,x)\) directly bottlenecks the accessible NTK corrections, giving rise to low-rank, biased learning, as formalized in the next section.

\paragraph{Freezing Parameters} Finally, note that \(\K\) captures the choice of dynamical parameters and \(\P\) the linearized model dynamics. Specifically, if we choose to fix the hidden weights, the core is reduced to the form \(V V^* = X X^*\), reflecting that only the input weights are trained. Likewise if we fix the input weights, the core becomes \(V V^* = H H^*\), the hidden state Gram matrix (motivating Figure~\ref{fig:regions}). In either case, however, \(\P\) does not change. On the other hand, if a model has the same mechanism for parameter entry but different linearized dynamics, the computational form of \(\K\) remains identical while \(\P\) changes. This shows that \(\P\) and \(\K\) can be modified effectively independently of one another, with very different roles.



\subsection{{Quantifying Self-Referential Bias Through Core Modes in V}}

Theorem~\ref{thm:core} shows that the \(\NTK\) is mediated by the
current weight-site state \(V\). Thus, gradient descent is not equally
responsive to all error directions: it learns most effectively through
modes represented in the model's current weight-site feature matrix.
{We refer to this measurable state-dependent filtering mechanism as \textbf{self-referential bias}: the current weight-site activity determines which adjoint directions survive the parameter core under Euclidean GD.}

The corollary below shows that \(\langle \NTK \err(h), \err(h)\rangle_F\), { the first-order rate of loss decrease induced by the GD step}, depends on how \(\mathrm{adj}(h)\), defined below, aligns with \(V\), over all time and trials. {When \(\P\) is empirically close to isotropic, \(\P^*\) does not strongly reweight error directions, so direct overlap with the weight-sites is informative. In this regime, alignment between the task target (or individual SVD modes of the target) and \(V\) provides an inexpensive diagnostic of full-NTK task alignment. By contrast, when \(\P\) is strongly anisotropic, its interaction with \(\K\) can further concentrate the spectrum and reduce the effective rank of \(\NTK\) and the effective range of learning, as explored in Section~\ref{subsec:regimes}.} This relationship is explored computationally in Experiment~\ref{subsec:grubias}. 

\setlength{\fboxsep}{.02\textwidth}
\noindent\fbox{%
    \parbox{.96\textwidth}{%
\begin{corr}[Self-referential bias via adjoint alignment]
\label{corr:adjointalign}
Let \(\P\), \(\K\), and \(\err\) be as in Proposition~\ref{prop:ntk}, and
define the global adjoint state by
\[
\text{adj}(h) := \P^*(\err(h)).
\]
If, moreover, \(\K = V V^* \otimes I_n\) as in Theorem~\ref{thm:core}, then
\[
\langle \NTK \err(h), \err(h)\rangle_F
= \|(V^* \otimes I_n)\text{adj}(h)\|_F^2.
\]
Thus, at each iteration, \(V\) acts as a filter, removing any components of \(\mathrm{adj}(h)\) orthogonal to \(V\), so that GD is restricted to { corrections with adjoints aligned with the weight-sites}.
\end{corr}
}}
\begin{proof}
This is a direct consequence of Theorem~\ref{thm:core}. Specifically, the NTK alignment expands as  
$$\langle \NTK \err(h), \err(h) \rangle_F = \langle \P (V V^* \otimes I_n)  \P^*(\err(h)), \err(h) \rangle_F \, .$$
We can pass \(\P (V \otimes I_n)\) to the right side of the inner product by taking its adjoint, yielding
$$ = \langle (V^* \otimes I_n) \P^*(\err(h)), (V^* \otimes I_n)\P^*(\err(h))\rangle_F \, .$$
Finally, from the definition of \(\text{adj}(h)\), this equals \(\|(V^* \otimes I_n) \text{adj}(h)\|^2_F\).
\end{proof}

{The local matrix--vector derivative behind this corollary is elementary. The global formulation makes the resulting restriction explicit as \(VV^*\otimes I_n\) across tasks, times, and state variables, allowing direct measurement of rank bounds, tensor modes, and target--NTK alignment. The experiments below then test interventions that alter these quantities, including input features, initialization, and structured preconditioning.}

\paragraph{Remarks} Corollary~\ref{corr:adjointalign} gives an operational form of
self-referential bias. The adjoint operator \(\P^*\) transports the error
signal into state-space, producing the global adjoint state
\(\text{adj}(h)\). Then, the Kronecker core measures the 
overlaps of this state with the span of the current weight-sites. The geometry of
\(V\) already determines which error directions can effectively drive
learning. If \(\text{adj}(h)\) has little overlap with the dominant modes of
\(V\), the corresponding first-order corrections are strongly suppressed, which can
produce slow learning or plateaus for those task components.

In the discrete recurrent example above, \(V=\mathrm{cat}(h,x)\). This means that if the joint hidden-input state is concentrated in a few dominant modes over time, trials, or hidden units, then GD is correspondingly biased to learn through those same modes. For example, if \(V\) is dominated by low-frequency dynamics over time and trials, high-frequency target components can receive very small first-order updates. If the activity of \(V\) drops off in a particular time period (e.g., in the response period in the Memory-Pro task below), target components concentrated in this range can likewise be strongly suppressed. Further, if \(V\) corresponds to low-rank latent dynamics in space, meaning that it can be projected to a low-dimensional PCA space, then corrections are constrained by the low-rank Gram matrix \(V V^T\), as we study in the next section. {This provides one concrete state-space mechanism related to prior observations of spectral/simplicity bias and low-dimensional learned representations~\citep{rahaman2019spectral, cao2021towards, Farrell2022,RecanatesiFarrell2021,farrell2023from,Freedman2006ExperienceDependent}. We do not claim that phenomena such as simplicity bias or neural collapse are identical to, or special cases of, the mechanism defined here; rather, the operator gives a quantitative way to test when restricted current activity contributes to such low-dimensional learning.} 


\subsection{Necessary Requirements for Low-Rank Learning (Over Space and Time)}
\label{subsec:spacetime}

We have refrained from vectorizing the global-state tensor \(h\) throughout. Here, we show that, in conjunction with Theorem~\ref{thm:core}, this allows us to define multiple reduced views of the NTK operator, capturing its dominant behavior over hidden units (space) or batches and timesteps (time). The core \(V V^* \otimes I_n\) from Theorem~\ref{thm:core} acts non-identically only on the temporal part of this view. Importantly, Proposition~\ref{prop:spacetime} shows that if \textit{either} the reduced spatial or temporal view of the NTK is low-rank, then the updates \(\delta h\) it produces are low-rank in \textit{both} axes. Thus, high-rank learning necessitates that both the spatial and temporal ranks of the NTK are high. We will undertake a concrete application of this in Section~\ref{subsec:regimes} below, where we explore distinct regimes in which either the spatial or temporal rank of the NTK is the limiting factor.

\paragraph{Space-Time Reduced Views of the Global-State NTK} Theorem~\ref{thm:core} shows that \textit{any} weight-based model exhibits a decomposition of the \(\NTK\) over the tensor product domain \(\R^{k} \otimes \R^{n}\). Here, \(k\) indexes the vectorized batch-time axis, \(k = n_x \cdot n_t\), while \(n\) is the physical neuron dimension. Using this as motivation, we call the left factor of such a tensor product the \textit{temporal part} and the right factor the \textit{spatial part}. Specifically, we define two reduced operators,
\begin{align}
\label{eqn:ntkview}
\NTK^{temp} := \frac{\text{tr}_n(\NTK)}{n}  \in \R^{k \times k} \, ; \quad \NTK^{space} := \frac{\text{tr}_k(\NTK)}{k} \in \R^{n \times n}\,.
\end{align}
where \(\text{tr}_n\) denotes the \textit{partial trace}, effectively summing along one axis of the operator. Note that \(\NTK\) is positive semidefinite, so this corresponds to a reduced second-moment view of the original global-state parameter sensitivity. In Appendix~\ref{sec:aprnla}, we detail how these operators can be implemented efficiently in a matrix-free way. These views of the NTK apply to any weight-based model covered by Theorem~\ref{thm:core}. Indeed, we compute these ranks for a non-recurrent self-attention transformer in Section~\ref{subsec:transformer} below.


\paragraph{Illustration} We first give a simple demonstration of the relationship between low-rank spatial and temporal properties via the operator \(\K\), which is always Kronecker separable by Theorem~\ref{thm:core}. Specifically consider a one-dimensional case where a single temporal mode dominates, \(\K = v v^T \otimes I_n\), where \(v \in \R^k\) is a vector, representing the same pattern of activity shared across all the joint hidden-input activity. Then, the operator is rank one temporally, but full rank spatially. Let \(\err(h) \in \R^{k \times n}\). Define \(q := v^T \err(h) \in \R^{n}\). Then,
\[
(v v^T \otimes I_n)(\err(h)) = v q^T \in \R^{k \times n},
\]
so every row of the output is a scalar multiple of \(q\). Thus, all updates lie in a one-dimensional spatial subspace, exploring only a single direction in \(\R^n\). In other words, the low-rank temporal structure of \(v v^T \otimes I_n\) implies a spatial bottleneck on every output of this operator. 

Next, we generalize and formalize this phenomenon in Proposition~\ref{prop:spacetime}.

\setlength{\fboxsep}{.02\textwidth}
\noindent\fbox{%
    \parbox{.96\textwidth}{%
\begin{proposition}[Space-Time \(GeNTK\) Bottleneck]
\label{prop:spacetime}
    In the context of Theorem~\ref{thm:core}, \(\NTK\) is an operator acting on \(\R^{k \times n}\), and we can define its reduced temporal and spatial views (Equation~\ref{eqn:ntkview}). Let \(\err(h) \in \R^{k \times n}\) denote any input error signal, and let \(\delta h = \NTK(\err(h))\) be the induced correction. Then,
    \[
    \rank(\delta h) \leq \min\!\bigl(\rank(\NTK^{space}), \rank(\NTK^{temp})\bigr),
    \]
    where \(\rank(\delta h)\) is the matrix rank of \(\delta h \in \R^{k \times n}\), bounding both its spatial and temporal expressiveness.
\end{proposition}
}}
\begin{proof}
See Appendix~\ref{subsec:stproofap} for a complete proof, relying primarily on the fact that the \(\NTK\), being a Gram operator, is positive semidefinite. In essence, we show that
\[
u^T \delta h = 0 \in \R^n \quad \forall u \in \ker(\NTK^{temp}).
\]
In turn, this implies that the range of corrections the NTK can produce must lie within the temporal image of \(\NTK^{temp}\). Since column rank and row rank of a matrix, e.g. of a correction \(\delta h \in \R^{k \times n}\), are identical, this yields the rank bound from \(\NTK^{temp}\). Similar reasoning yields the bound from \(\NTK^{space}\). 
\end{proof}

\paragraph{Remarks}
This proposition shows that biased low-rank learning emerges when \textit{either} of the two partially reduced operators has low-rank. A low-rank NTK at a particular iteration of GD implies that learning may be highly biased toward particular targets or low-complexity updates, such as failing to capture high-frequency temporal structure in the target of a recurrent task. In the special case where \(\P\) is explicitly separable across space and time, meaning that propagation can be written as a tensor product of independent spatial and temporal operators, this yields the exact bound
\[
\rank(\delta h) \leq \rank(V V^*) \, .
\]
showing that the spatial and temporal complexity of updates are exactly bottlenecked by the Kronecker core \(V V^*\). Even outside of this case, the spatial and temporal partial views, along with their associated singular values and ranks, can be computed efficiently in our framework (Appendix~\ref{sec:aprnla}). Finally, we remark that high-rank latent representations may still emerge over the course of GD, even when updates are low-rank at every iteration. Only under additional simplifying assumptions that allow the NTK to remain controlled over training (e.g., lazy regimes or linear models) can such a conclusion be made, since it requires tracking the full course of latent formation across all GD iterations (see Future Directions). 



\section{The Kronecker Core Predicts Learning Bottlenecks Across Architectures}


\subsection{GRU Joint Hidden-Input Activity Predicts Dominant GeNTK Modes}
\label{subsec:grucore}

\begin{figure}[t]
    \centering
    \includegraphics[width=.98\linewidth]{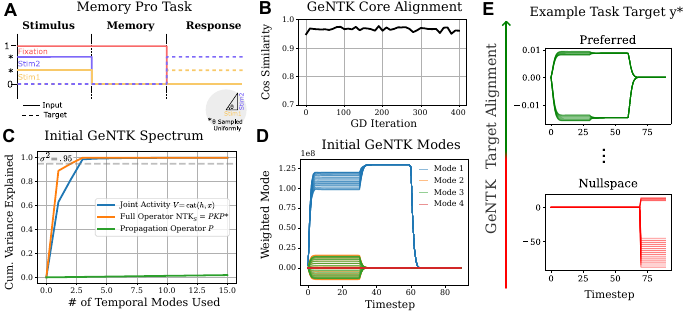}

    \caption{\textbf{Temporal bottlenecking of the global-state NTK by the Kronecker core \(V V^T\).}
\textbf{A} Schematic of the Memory-Pro task, in which the model must reproduce a two-dimensional stimulus after a delay period.
\textbf{B} Cosine similarity between the core, \(\K = VV^* \otimes I_n\), and the NTK, \(\text{cos}(\NTK, V V^T \otimes I_n)\), over GD training of the GRU model on the task in A, showing that the two operators share a similar common basis. Here, we use a model with Xavier initialization, corresponding to Network 1 in Figure~\ref{fig:memproinit} below.
\textbf{C} Cumulative variance explained by the temporal spectrum of the full NTK, the Kronecker core \(V V^T\), and the propagation operator \(\P\). The full NTK has effective rank close to that of \(V V^T\), while \(\P\) remains comparatively high-rank in this regime.
\textbf{D} Dominant temporal modes of the NTK. Each individual mode is a matrix \(\R^{n_x \times n_t}\), describing the time and trial structure of a particular input to the operator (formally defined in Section~\ref{subsec:spacetime}).  Thus, a single mode consists of $n_x=20$ here, as shown by the multiple curves of each color in the plot. Importantly, the NTK modes encode the structure of the input and hidden state, dropping to zero in the response period (timesteps 60-90 here).
\textbf{E} Example one-dimensional task targets in \(\R^{n_x \times n_t \times 1}\), over batches and time. The top is an example of a task that would highly align with the NTK, i.e., its activity is well-explained by the dominant modes in D. By contrast, the second target is badly aligned, since it primarily resides in the range where the NTK modes are zero. Interestingly, the Memory-Pro task has a target similar to the bottom row since it requires a delayed response. Section~\ref{subsec:grubias} explores the consequences of this for learning, showing that this particular GRU is unable to learn the full Memory-Pro task with SGD.}
    \label{fig:skree}
\end{figure}

\paragraph{GRU Global-State Core and GeNTK}
For the GRU with \({n_h}\) neurons, the core of the NTK is \(V V^T \otimes I_{3n}\), where \(V\) consists of the concatenated hidden-state and input activity,
\[
V = \mathrm{cat}(h, x) \in \R^{n_x \times n_t \times ({n_h} + n_{in})}\, ,
\]
and \(3n\) comes from the three-fold repetition of weight-sites in this model (derivation in Appendix~\ref{subsec:corederivap}). Figure~\ref{fig:skree}A illustrates the Memory-Pro, in which a two-dimensional stimulus is provided to the model during an initial stimulus period, it must retain knowledge of this stimulus during a memory period and finally should respond with the identical stimulus provided during a final, response period. Here, the stimulus has the form \(\cos(\theta), \sin(\theta)\) for uniform angle \(\theta \in [0, 2 \pi]\), so the latent dynamics the model forms on this task often resembles an attractor to a two-dimensional ring~\citep{driscoll2024flexible}. Finally, the model is provided with a fixation input, which is set to \(1\) during the stimulus and memory periods and switches to zero during the response period, indicating that the model should respond. 

\paragraph{Findings} Figure~\ref{fig:skree}B measures the alignment between the \(\NTK\) operator and its core, \(V V^* \otimes I_{3n}\), over multiple iterations of GD on the Memory-Pro task. It shows that, throughout training on this task, the core provides a good predictor of the spectrum of the full NTK, meaning that it captures the dominant modes of the NTK. This point is further reinforced by Figure~\ref{fig:skree}C, which compares the spectrum of \(\NTK\) to those of the operators \(V V^*\) and \(\P\) from which it is composed, at the outset of GD training. Specifically, we measure the dominant \textit{temporal modes} of the operators involved, as defined in Section~\ref{subsec:spacetime}. These modes can be viewed as matrices in \(\R^{n_x \times n_t}\), describing, for each distinct task input, the dominant temporal structures to which the operator responds. Using this perspective, we observe that \(\P\) is effectively full-rank, with singular values that are relatively evenly distributed across the range of outputs it can produce. In contrast, the core \(V V^*\) is low-rank (with only 3 modes needed to capture 95\% of its variance). Intuitively, this suggests that \(\P\) is roughly isotropic and does not strongly bias the direction of corrections. On the other hand, the Kronecker core \(V V^*\) strongly biases corrections, acting highly anisotropically. Finally, we find that \(\NTK\) is itself low-rank, with its cumulative variance curve lying strictly above that of \(V V^*\) and modes well predicted by this Kronecker core matrix. 


Figure~\ref{fig:skree}D shows that the leading temporal modes of the NTK are highly structured, concentrating on a small family of response profiles across the trial. These modes are not generic basis functions; they are induced by the task-dependent activity present at initialization. Consequently, even before training, gradient descent is constrained to act most strongly along a narrow subspace of temporal directions selected by the current recurrent activity.


\begin{figure}[t]
    \centering
    \includegraphics[width=1\linewidth]{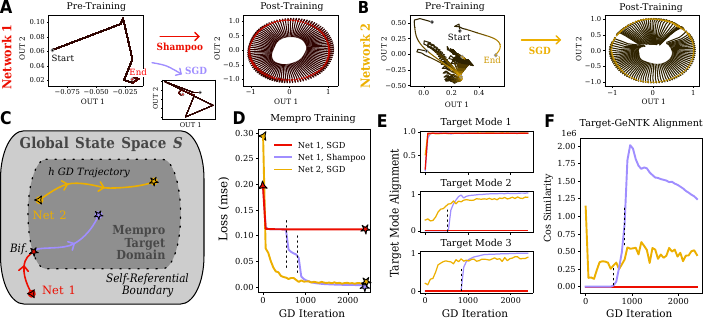}
    \caption{\textbf{Self-referential bias can stall SGD on the Memory-Pro task.}
    \textbf{A-B} Outputs before and after training for two GRU initializations. \textbf{A} (Network~1): a default Xavier initialization with weight scale \(1\), whose hidden dynamics collapse to a single fixed point under zero input during the response period, regardless of the input (``End'' in left panel of Figure). \textbf{B} (Network~2): as an illustrative case, an initialization chosen to produce \(5\) non-trivial fixed points under zero input (Appendix~\ref{subsec:gruexap}), showing that the task-driven dynamics have multiple endpoints at the final evaluation timestep. As in further panels, this initialization leads to better GeNTK-target alignment initially. Thus, the two networks begin with qualitatively different latent dynamics before training.
    {\textbf{C} Schematic of the corresponding training trajectories in global state space \(S\).} Network~2 begins in a regime already containing the qualitative structure needed for the task, whereas Network~1 must first create this structure during training.
    \textbf{D} Training losses for the two initializations. Network~2 trains successfully with SGD, while Network~1 plateaus under SGD and only solves the task when trained with the {structured matrix-preconditioning method} Shampoo~\citep{gupta2018shampoopreconditionedstochastictensor}. In the latter case, a transition occurs around iteration \(750\), where new non-trivial hidden-state structure emerges.
    \textbf{E} Alignment of the task target with its leading three temporal modes during training. Network~2 begins with non-trivial alignment in all three modes and learns them steadily with SGD. By contrast, Network~1 under SGD learns only the first mode, while the higher modes remain effectively absent. Training Network~1 with Shampoo eventually produces these missing modes.
    \textbf{F} Cosine alignment of the GeNTK with target mode \(3\) during training. For Network~1, the alignment is initially nearly zero, consistent with the near-null target modes identified in Figure~\ref{fig:skree}E. Under Shampoo, GeNTK-target alignment rises before the corresponding hidden-state mode becomes clearly visible in \textbf{E}, consistent with the operator first becoming aligned with the missing target mode and then driving it into the dynamics.}
    \label{fig:memproinit}
\end{figure}

\subsection{Self-Referential Bias Stalls GD When Target Modes are Absent in The Initial GeNTK}

\label{subsec:grubias}

\paragraph{Self-Referential Bias}
As suggested by the bottom row of Figure~\ref{fig:skree}E, some task targets can have weak alignment with the dominant temporal modes of \(\NTK\) at initialization, meaning that their projection onto the high-eigenvalue eigenspaces is small. In this subsection, we examine this in detail and show that such NTK-target misalignment explains the self-referential bias predicted by Corollary~\ref{corr:adjointalign}: because GD corrections are constrained by the current hidden and input activity, learning is initially biased toward target components already represented in the model dynamics, while poorly aligned components receive only very small updates and therefore learn slowly, producing plateaus.

\paragraph{Two GRU Initializations}
To study this, we contrast two GRUs trained on the Memory-Pro task. Network~1 is the default initialization used in the previous subsection. At initialization, when the input is removed during the response period, its hidden dynamics collapse stably to the trivial fixed point at zero (Figure~\ref{fig:memproinit}A, pre-training). Network~2 is initialized so that, under zero input, the model instead exhibits five non-trivial stable fixed points (dynamics on the task in Figure~\ref{fig:memproinit}B pre-training; construction in Appendix~\ref{subsec:rnnexap}). Thus the two models begin with qualitatively different latent dynamics, and consequently with different weight-site activity \(V = \cat(h,x)\), even before training.

Throughout this subsection, we measure alignment with the target modes obtained from SVD on the target matrix. Specifically, the target has shape \(n_x \cdot n_t \times 3\), due to the three outputs, so its left singular vectors define three orthogonal modes over time and batches, \(U \in \R^{n_x \cdot n_t \times 3}\).

\paragraph{Network~1 Fails to Learn Some Target Modes}
{The distinction in initial latent dynamics is important because, when \(\P\) is well-conditioned, the previous subsection shows that the dominant temporal modes of the NTK are largely determined by the activity-induced core \(V V^T\).} For the Memory-Pro task, the target lies mainly in the response period, whereas the external stimulus is only present earlier in the trial. Hence, if the hidden state collapses during the response period, then the target modes have almost no overlap with the temporal modes induced by \(V\). This is exactly what happens in Network~1: both the input and hidden activity are nearly absent during the response period. Thus, the target is initially very poorly represented in the Kronecker core, \(V V^T\), hence not well represented in the spectrum of \(\NTK\). {By contrast, Network~2 maintains non-trivial hidden activity during this period, and \(V\) has substantial overlap with all three target modes.} This in turn improves NTK-target alignment from the outset.

Training dynamics in Figure~\ref{fig:memproinit}D-F reflect the distinction. Network~2 learns the task steadily with SGD, while Network~1 quickly reaches a plateau. Figure~\ref{fig:memproinit}E shows the reason. In all cases the first target mode is learned, but for Network~1 under SGD the higher target modes remain essentially absent. {In other words, early SGD progress is concentrated in the target modes with appreciable gain under the current \(\NTK\) geometry, while the poorly aligned modes change only weakly. This is the local effect quantified by self-referential bias.}

Figure~\ref{fig:memproinit}F makes this still more explicit and tracks the minimum NTK-target alignment throughout training. For Network~1, NTK-target alignment is effectively zero, consistent with the near-null target modes identified in Figure~\ref{fig:skree}E. As a result, the corresponding corrections produced by SGD are extremely small, and training stalls, plateauing after the first mode is learned. In contrast, Network~2 begins with substantially higher target alignment for all three target modes, and is therefore trainable by SGD alone.

\paragraph{{Changing the Update Geometry with Structured Preconditioning}}
To test whether this plateau is really caused by poor NTK-target alignment, rather than by a generic limitation such as insufficient training time, poor learning-rate tuning, or an inaccessible parameter-space solution, we also train Network~1 with Shampoo~\citep{gupta2018shampoopreconditionedstochastictensor}. {Shampoo adaptively rescales parameter updates using structured matrix preconditioning. We therefore interpret this experiment narrowly as testing whether a changed parameter-space update geometry can access directions that are poorly conditioned under Euclidean GD. In this example, the structured preconditioner is sufficient to escape the observed plateau.} Around iteration \(750\), a new non-trivial mode emerges in the hidden dynamics, after which the remaining target modes become learnable (Figure~\ref{fig:memproinit}D-F). {Notably, the $\NTK$-target alignment begins to rise before this new activity is clearly visible in the hidden dynamics, consistent with the effective update geometry becoming better aligned with the missing target mode before that mode is expressed strongly in the latent state. In Appendix~\ref{subsec:gruexap} we compare two additional optimizers as approaches for overcoming the plateau emerging from the badly conditioned $\NTK$ here, specifically Muon and Adam~\citep{jordan2024muon,kingma2014adam}.} 

This experiment illustrates one concrete consequence of self-referential bias. When the initial hidden dynamics already contain activity aligned with the task, the corresponding core \(V V^T\) and NTK are well conditioned for learning, and SGD succeeds. When these modes are absent, the relevant target directions are effectively null under the NTK, producing a long plateau or complete failure of iterative learning under SGD. Thus, there is a \textbf{trade-off} between structured low-rank dynamics making models stable and interpretable \citep{farrell2023from, mastrogiuseppe2018linking, ostojic2024computational}, but also biasing learning toward those same directions, restricting the range of tasks that SGD can learn efficiently from a given initialization. {This mechanism is also consistent with observations that fixed points and dynamical motifs can be reused across related tasks~\cite{driscoll2024flexible}: restricted current dynamics can bias first-order updates toward modifying or reusing already represented modes, although establishing this as a general explanation of motif reuse is beyond the present work.}

\subsection{Low-Rank Space-Time Learning Regimes Based on Rank of P and K}

\label{subsec:regimes}
{Throughout this subsection, ``rank collapse'' refers to effective rank of the full operator or its reduced views, not loss of algebraic invertibility of \(\P\), which remains full-rank wherever Proposition~\ref{prop:ntk} applies.}
We further examine the low-rank structure of \(\NTK\), to isolate distinct regimes in which gradient descent is bottlenecked in its rank over space, time or both. Specifically, we compute the effective rank of the partially reduced views, \(\NTK^{temp} \in \R^{n_x \cdot n_t \times n_x \cdot n_t}\) and \(\NTK^{space} \in \R^{{n_h} \times {n_h}}\), defined in Section~\ref{subsec:spacetime}, for a vanilla RNN with global state space \(S = \R^{n_x \times n_t \times {n_h}}\), over a batch of \(n_x\) inputs, simulated for \(n_t\) timesteps, with \({n_h}\) hidden units.

\begin{figure}[t]
    \centering
    \includegraphics[width=0.95\linewidth]{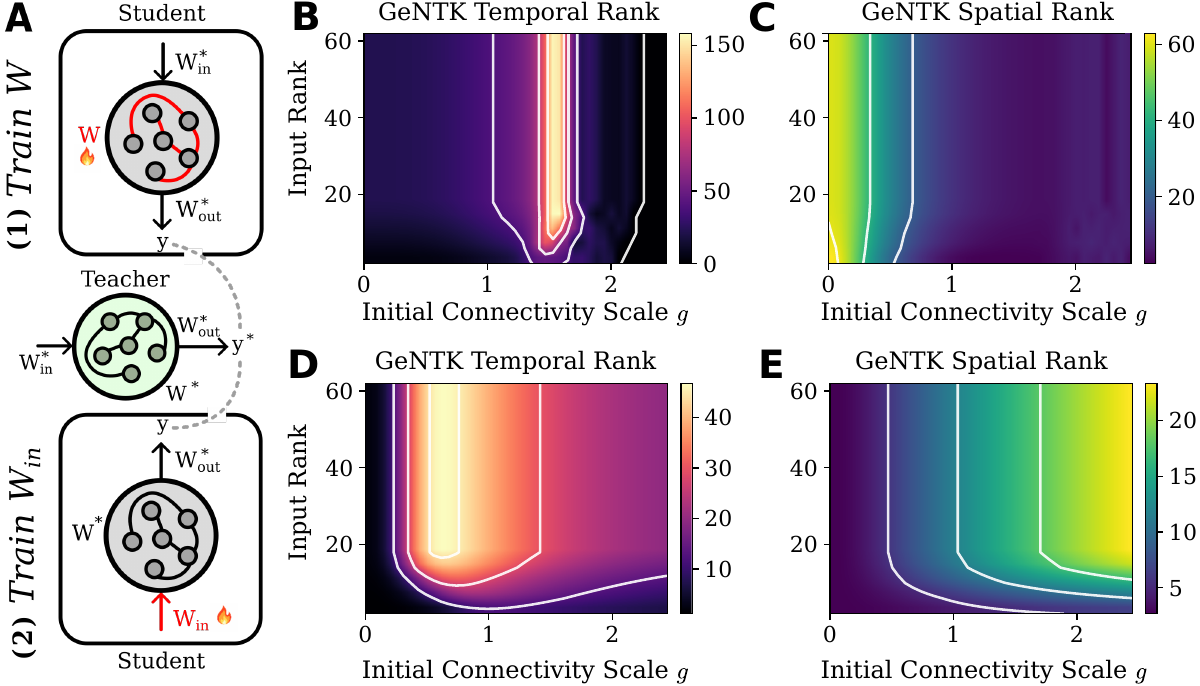}
    
    \caption{\textbf{Recurrent gain and input rank induce distinct spatial and temporal NTK bottlenecks.}
    \textbf{A} Two variants of a student-teacher task. For both, we begin with a vanilla RNN teacher with fixed weights \(W^*, W_{in}^*, W_{out}^*\), all drawn from Xavier normal initialization. In the top task, we freeze the student to have identical weights other than \(W\), which is initialized randomly with gain \(g\) and trained with GD. The training objective is that the outputs of the two models on random normal inputs match. In the bottom task, we instead fix \(W = W^*\) and randomly initialize \(W_{in}\) with gain \(g\), with the same objective. Thus, these tasks isolate the role of learning recurrent versus input weights in GD, with the NTK core defined by hidden or input state, respectively (see example in Section~\ref{subsec:core_rnn_example}). The inputs to each model are random normal (see Appendix~\ref{subsec:rnnexap}), with an input channel count varied in B--F. 
    \textbf{B} Temporal rank when training the recurrent weights. This rank increases with gain \(g\) before collapsing again in the unstable regime, consistent with the core prediction that richer recurrent activity increases the temporal rank available to the NTK.
    \textbf{C} Spatial rank when training the recurrent weights. This rank, determined primarily by the propagation operator \(\P\), collapses as \(g\) increases, corresponding to \(\P\) concentrating onto a small number of unstable or exploding modes.
    \textbf{D-E} Temporal and spatial rank when training the input weights. Here, the temporal rank grows with input rank, consistent with the bottlenecking effect predicted by Theorem~\ref{thm:core}, while the spatial rank varies comparatively little.
}
    \label{fig:regions}
\end{figure}

The spectrum of \(\NTK^{space}\) describes the global spatial bias of the operator. If it is concentrated in a few directions, then regardless of the error signal, updates are confined to a low-dimensional subspace of \(\R^{n_h}\). In contrast, the spectrum of \(\NTK^{temp}\) describes the range of temporal modes available for solving tasks with a given set of inputs. When \(\NTK^{temp}\) is low-rank, updates are restricted to a small family of temporal signals, causing bottlenecking of spatial corrections induced by a given error (Proposition~\ref{prop:spacetime}). Thus, the spatial effective rank captures biases introduced by the propagation geometry of the architecture, while the temporal effective rank captures biases introduced by the activity and input patterns available during learning.

In particular, Figure~\ref{fig:regions} studies these reduced operators in a student-teacher RNN setting, summarized in panel A. We consider two complementary tasks. (1) We set all student weights identically to the teacher other than the recurrent weights, which must be trained with GD so the student and teacher outputs match. (2) Instead, the task requires training of only the input weights, with all other student weights identical to the teacher weights and frozen during training. In both tasks, we investigate initializing the dynamic weights with a random normal distribution and gain \(g\).  
In addition to sweeping the gain term, we also sweep the effective rank of the task inputs. These two sweeps are chosen to probe the two factors in the decomposition, \(\NTK = \P (\cat(H,X)\cat(H,X)^T \otimes I_n)\P^*\), demonstrating that the core \(\cat(H,X)\cat(H,X)^T\) primarily bottlenecks temporal structure, while the propagation operator \(\P\) reshapes the spatial structure.

\paragraph{Temporal Bottlenecks from the Core}
Figure~\ref{fig:regions}B shows that when training the recurrent weights, the temporal rank of the NTK increases as the recurrent gain \(g\) moves away from a strongly contractive regime, consistent with richer hidden activity producing a higher-rank core \(V V^T\) and hence a more expressive temporal NTK~\citep{sompolinsky1988chaos, rajan2006eigenvalue,engelken2020lyapunovspectrachaoticrecurrent}. However, this increase is confined to an intermediate regime: for larger \(g\), the temporal rank drops sharply, indicating that the temporal richness induced by the core is ultimately lost once propagation becomes sufficiently ill-conditioned. Thus, increasing gain can initially enrich the temporal geometry available to learning, but beyond a critical regime instability produces a new low-rank bottleneck.

The same mechanism appears when training the input weights. Figure~\ref{fig:regions}D shows that the temporal rank grows strongly with input rank, as predicted by the core decomposition, since the temporal bottleneck is controlled primarily by the rank of the weight-site activity \(V = \cat(H,X)\). The dependence on recurrent gain is present but secondary compared with the effect of input rank.

\paragraph{Spatial Bottlenecks from Propagation}
The spatial rank behaves differently. Figure~\ref{fig:regions}C shows that when training the recurrent weights, the spatial rank collapses as \(g\) increases. This collapse is not explained by the Kronecker core, which is separable and isotropic over \(\R^{n_h}\). Rather, it is governed primarily by the operator \(\P\), which becomes increasingly ill-conditioned as the dynamics approach instability. In this regime, propagation concentrates onto a small number of expanding directions, so that the NTK becomes effectively low-rank in space even when the temporal core remains expressive.

By contrast, Figure~\ref{fig:regions}E shows that the spatial rank varies more gradually, increasing somewhat with input rank but without the sharp transitions seen in the temporal rank. This reinforces the distinction above: input complexity primarily controls the temporal bottleneck through \(V V^T\), while recurrent gain primarily controls the stronger spatial bottleneck through \(\P\).

\begin{figure}[t]
    \centering
    \includegraphics[width=.8\linewidth]{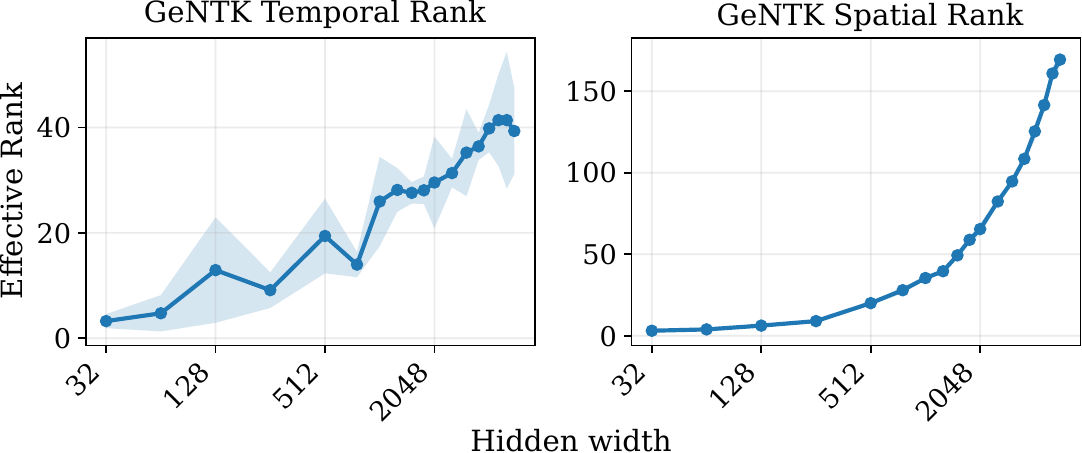}
    \caption{{\textbf{$\NTK$ rank under width scaling.} Using the student--teacher setup of Section~\ref{subsec:regimes} with both recurrent and input weights trainable and \(g=1\), temporal effective rank saturates at large width while spatial effective rank continues to grow. Curves show the mean over three seeds. Shaded regions denote variation across seeds.}}
    \label{fig:rnnwidth}
\end{figure}

Taken together, these sweeps reveal two distinct bottlenecks in the empirical NTK. The weight-site core \(V V^T\) primarily limits the temporal diversity of available corrections, while the propagation operator \(\P\) primarily limits how broadly those corrections are distributed across state space. From this perspective, increasing recurrent gain can enrich temporal expressivity, but only until propagation becomes sufficiently ill-conditioned that updates concentrate onto a small number of directions~\citep{sompolinsky1988chaos, rajan2006eigenvalue,engelken2020lyapunovspectrachaoticrecurrent}. Favorable regimes for learning a diverse set of tasks therefore require a core rich enough to represent task-relevant corrections and a propagator \(\P\) that remains sufficiently well-conditioned to distribute them broadly.  In a broad sense, this echoes the influential findings of \citep{sussillo2009generating}, regarding optimal learning regarding at the ``edge of chaos.''

\subsection{{RNN Width Scaling}}
\label{subsec:widthsweep}

{
\paragraph{Dependence of Temporal and Spatial Rank with Model Scale} To test how these reduced ranks behave as the network is widened, we repeat the RNN setup of Section~\ref{subsec:regimes} at \(g=1\), now treating both recurrent and input weights as trainable parameters. As hidden width increases from \(32\) to \(5632\), the spatial effective rank continues to grow, while the temporal effective rank increases initially and then appears to asymptote near \(40\) (Figure~\ref{fig:rnnwidth}). Thus, increasing width adds spatial degrees of freedom without removing the temporal bottleneck of the task-conditioned NTK. Additional controls in Appendix~\ref{subsec:rnnwidthcontrols} show that this saturation is not a simple finite-sample or joint-activity-rank ceiling. This observation is consistent with physical rank depending primarily on \(\P\), the propagation operator, and temporal rank on \(\K\), the bottle-neck, in this RNN case. At initialization, \(\P\) is roughly isotropic and scales with the physical neuron count, while \(\K\) scale with the effective dimension of the dynamics and task inputs. Importantly, Proposition~\ref{prop:spacetime} implies that low-rank corrections emerge if \textit{either} of the spatial or temporal rank are low, so the plateau in the temporal rank (i.e., diversity of SVD modes in the temporal partial trace of $\NTK$) implies the $\NTK$ will produce similarly low-rank correction for very large models. Practically, these results validate that one can investigate the spectral properties of $\NTK$ even for large models, including about 2.5 billion parameters in the largest model here. 
}

\subsection{The GeNTK Core Bottleneck Appears in a Self-Attention Transformer}

\label{subsec:transformer}

\begin{figure}[t]
    \centering
    \includegraphics[width=\linewidth]{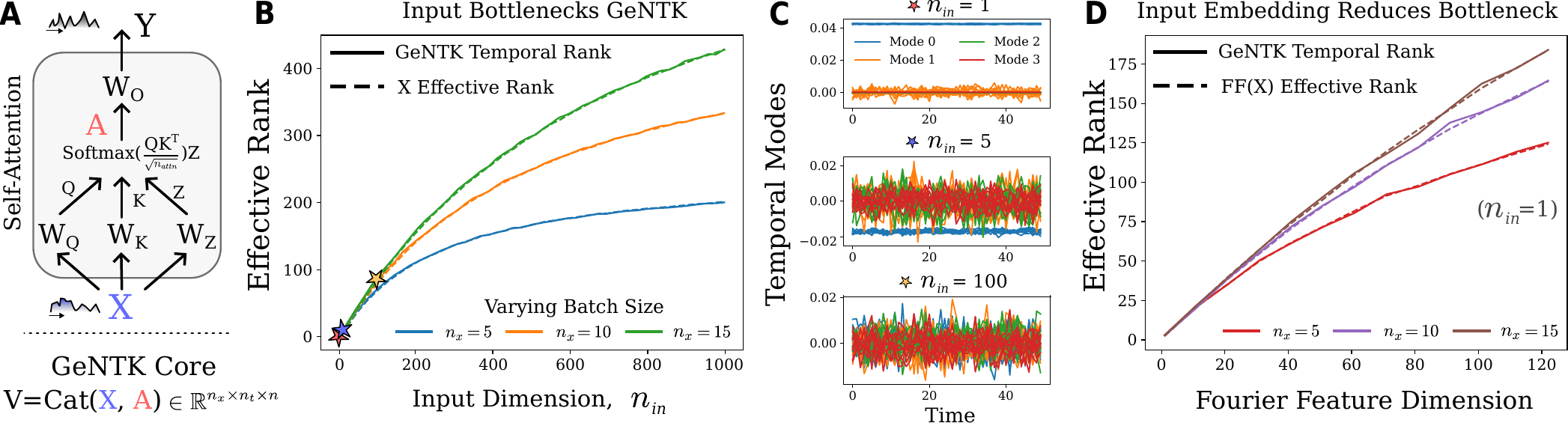}

    \caption{\textbf{Rank bottlenecking by input dimension in a self-attention model.}
    \textbf{A} Self-attention architecture with time-varying inputs \(X \in \R^{n_x \times n_t \times n_{in}}\). As in the main text, the weight-site core of the NTK consists of the concatenated input activity and attention matrix, \(V = \cat(X,A)\). See Appendix~\ref{subsec:transexap} for model and input details.
    \textbf{B} Because \(A\) lies in the same temporal span as \(X\), the temporal rank of the NTK is bottlenecked by the batch size and input dimension. Increasing the input dimension increases the NTK temporal rank.
    \textbf{C} Three examples with distinct input dimension and fixed batch size \(n_x = 10\), showing the variance-weighted dominant temporal modes of the NTK, \(V_j \in \R^{n_x \times n_t}\). For one-dimensional inputs, there is only a single substantially non-zero mode, which is approximately constant over time. By contrast, for \(n_{in}=100\), there are many non-zero modes with a much flatter spectrum. Each mode is plotted over all \(10\) batch inputs, with color indicating mode identity. Here \(X \sim \mathcal{N}(0,1)\) is drawn randomly over time, batch, and input dimension, but the rank bottleneck applies for arbitrary task inputs.
    \textbf{D} Applying an input embedding (deterministic Fourier features in this case) maps a one-dimensional input signal into a higher-dimensional feature space. Increasing the number of features increases the NTK temporal rank (all of the curves correspond to $n_{in}=1$, at the extreme left of the curves in panel $B$), reducing the bottleneck and yielding a more expressive range of GD updates.}
    \label{fig:transformer}
\end{figure}

Our analysis is not confined to recurrent models, although they are the primary focus of this work. Indeed, Theorem~\ref{thm:core} relates to any model that can be written as a system of constraints relating the state variables, inputs and weight-based parameters. This encompasses a wide range of models, including any explicit, weight-based model. To illustrate this structural theorem in a distinct setting, we apply the same perspective to a nonlinear self-attention block of a transformer. We use this example to show that the same Kronecker-core mechanism appears in a non-recurrent model as well: low-dimensional input structure can bottleneck the \emph{temporal} rank of the NTK, while widening the attention block primarily affects its spatial and overall rank. We note that the same derivation extends to a full transformer block with attention and MLP layers; for simplicity, we focus here on the attention component alone, with the full derivation deferred to the Appendix~\ref{subsec:corederivap}.

Consider a self-attention block applied to a signal with \(n_t\) timesteps, so the input has the form
\[
X \in \R^{n_x \times n_t \times n_{in}},
\]
where \(n_{in}\) is the input dimension and \(n_x\) is the batch size. The attention block first computes
\[
Q = X W_Q^\top, \qquad
K = X W_K^\top, \qquad
Z = X W_Z^\top,
\]
where we use \(Z\) for the value channel to avoid confusion with our weight-site notation \(V\). The attention output is then
\[
A = \operatorname{softmax}\!\left(\frac{QK^\top}{\sqrt{n_{attn}}}\right) Z,
\qquad
Y_{\mathrm{attn}} = A W_O^\top,
\]
where \(Q,K,Z,A \in \R^{n_x \times n_t \times n_{attn}}\).

\paragraph{Input Dimension Produces a Temporal Bottleneck}
The trainable weights are \(W_Q, W_K, W_Z,\) and \(W_O\), so the corresponding weight sites are
\[
V = \cat(X,X,X,A).
\]
Thus, by Theorem~\ref{thm:core}, the \(\NTK\) core again has the form \(V V^T \otimes I_n\), which can be viewed as a matrix of shape \(n_x \cdot n_t\) by \(3 n_x + n_{attn}\). Naturally, this fits Theorem~\ref{thm:core}, with \(k = n_x \cdot n_t\) constituting the \textit{temporal dimension} of \(\NTK\), with partially reduced operator \(\NTK^{temp}\) indicating dominant modes of this operator as signals over batches and input timesteps, and \(n = 3 n_x + n_{attn}\) the \textit{spatial dimension} of the operator. In particular, since \(Q\), \(K\), and \(Z\) are all generated linearly from \(X\), and \(A\) is formed by reweighting temporal content derived from these quantities, the dominant temporal modes available to \(\NTK\) are strongly constrained by the temporal span already present in \(X\).

This has the consequence that if the task input is low-dimensional, then the temporal rank of the NTK is necessarily low as well. Indeed, Figure~\ref{fig:transformer}B shows that as the input dimension increases, the temporal rank of the NTK increases with it. Figure~\ref{fig:transformer}C gives a more concrete view of this effect. For one-dimensional time-varying inputs, the NTK has essentially a single dominant temporal mode, which is nearly constant across time. For higher-dimensional inputs, many more temporal modes are available, and the spectrum becomes less concentrated. Additional sweeps in Appendix~\ref{subsec:aptrans} further clarify this distinction: while the temporal rank depends primarily on the input representation, the spatial rank and full operator rank grow strongly with attention width (Appendix Figure~\ref{fig:twodaptrans}).

From the operator viewpoint, this means that gradient descent is not free to produce arbitrary temporal corrections in the attention model. Rather, it is constrained to update along the dominant temporal modes already present in the weight-site activity. If the input consists only of a narrow family of temporal signals, then the NTK is biased toward that same family. In particular, low-dimensional inputs induce a low-rank temporal bottleneck, so the resulting updates are confined to a restricted set of temporal patterns. This is a similar self-referential mechanism seen in the recurrent case, now arising from the representation bottleneck of the attention block rather than from recurrent dynamics. At the same time, widening the attention block can still enlarge the spatial and overall rank of the operator, so the bottleneck here is specifically temporal rather than a claim that the full transformer NTK is globally low-rank.

\paragraph{Input Embeddings Relieve the Bottleneck}
In Figure~\ref{fig:transformer}D we show that the bottleneck can be relieved on a task of fixed input dimension by enriching the input representation before attention is applied. In our experiments, we use deterministic Fourier features to map a one-dimensional input into a higher-dimensional feature space. As the number of features increases, the temporal rank of the NTK increases as well, yielding a broader and more expressive range of GD updates. This provides a mechanistic explanation, within our Kronecker-core framework, for why richer input embeddings or positional features can substantially improve learnability in attention-based models. It is consistent with and effectively extends prior work, which used the infinite-width NTK for MLP models to show that Fourier features produce a better conditioned NTK for learning in this context, avoiding the inductive bias of GD toward learning low-frequency features of the target~\citep{tancik2020fourierfeaturesletnetworks}. By contrast, adding multiple attention heads in this toy setting does not substantially change the same temporal bottleneck, consistent with the fact that the heads still derive their query, key, and value channels from the same underlying input representation (Appendix Figure~\ref{fig:apmultihead}).

Taken together, these results show that the Kronecker-core perspective extends beyond recurrence and sheds light on one reason richer input encodings can improve conditioning in transformer-like models. In a self-attention block, the temporal rank of \(\NTK\) is bottlenecked by the structure of the input representation appearing in the weight sites, and gradient descent is correspondingly biased toward the dominant temporal modes of that representation. Enriching the input features lifts this temporal bottleneck and produces a more expressive temporal NTK, while widening the attention block primarily enriches the spatial and overall operator structure. {Thus, in this non-recurrent attention setting, the core decomposition identifies a temporal bottleneck that is relieved by enriching the input representation.} 

\section{Discussion}

{We introduced a finite-width framework for analyzing gradient descent directly in the full model state space through the empirical NTK acting on the state tensor, \(\NTK\). From an implicit reformulation of the model dynamics, we obtain
\[
\NTK = \P \K \P^*,
\]
which separates propagation through the model dynamics (\(\P\)) from immediate parameter-to-state interactions (\(\K\)).} For a broad class of weight-based models, we then show that \(\K\) has an exact Kronecker-core structure determined by forward-computable weight-site activity. This gives a tractable and interpretable handle on the state-space geometry of GD in the typical setting where \(\NTK\) is otherwise too large to study directly. {In particular, it provides a quantitative description of local mode-selective bias: under Euclidean GD, the weight-site core filters adjoint directions according to their overlap with the current activity. We refer to this state-dependent filtering as self-referential bias.}

This perspective is useful both theoretically and practically. Theoretically, it gives a concrete language for describing how model structure, dynamics, and representation geometry constrain learning. Within the implicit formalism, it is natural to compute \(\NTK\) for larger composed models by expanding the state or parameter spaces. One can also incorporate simplifying assumptions into the model by lifting the global constraint into a larger system of constraints {that make the state or parameter structure more explicit.} Then, computing \(\P\) and \(\K\) and projecting the NTK back to the original global state space yields an NTK reduced according to the assumption. This procedure is exactly used in the proof of Theorem~\ref{thm:core}. Finally, the framework makes the form of \(\P\), as the inverse of a one-step generator, and \(\K\), as a weight-site Gram matrix, explicit, allowing one to reason directly about how modifying dynamics or parameterization changes the bias of GD.

Practically, the approach provides a tool for diagnosis of bottlenecks. {It can identify task directions that are poorly aligned with the local learning geometry and distinguish spatial from temporal bottlenecks at a given iteration. The RNN width sweep further shows that these reduced views remain informative as width grows: temporal effective rank saturates while spatial effective rank continues to increase.} More broadly, the framework provides a way to study internal representation learning in finite-width networks without focusing the analysis on system outputs or needing to take infinite-width limits.  The accompanying matrix-free toolkit, \href{https://github.com/meeree/gentk}{gentk}, makes these questions computationally accessible in real models. {To test practical observability beyond the small mechanistic examples, we additionally validate the decomposition and weight-site core on pretrained Pythia transformer language models up to 410M parameters (Appendix~\ref{subsec:pythia}).}

{The analysis in this paper is iteration-local: \(\NTK\) describes the first-order state correction induced by a GD step at a fixed state--parameter tuple. We do not integrate a continuous-time optimization trajectory or claim that the operator remains fixed across training; understanding its evolution under GD, SGD, or continuous-time limits is outside the scope of the present work. The same model-side factors can nevertheless be reused to describe alternative local update geometries, as made explicit below. Tracking their evolution over a full optimizer trajectory remains a separate problem and an important direction for future work~\citep{gu2024mambalineartimesequencemodeling, liu2021linearitylargenonlinearmodels, clark2026structure}.}

{For example, if a single update is locally represented by a parameter-space preconditioner \(M\), then
\[
\delta h=-\eta\,\P\bigl(D_\theta\F\,M\,(D_\theta\F)^*\bigr)\P^*\err(h).
\]
Euclidean GD corresponds to \(M=I\); natural-gradient, Gauss--Newton, and structured preconditioning methods instead modify the parameter-space metric, with the precise form depending on the loss and optimizer. History-dependent or nonlinear optimizers need not reduce to a fixed \(M\), so this equation is a local template rather than a claim about their complete training dynamics.

Note that the building blocks here also naturally extend to describing other learning rules. For example, truncated backpropagation rules can be seen as replacing \(\P^*\), the operator producing backpropagated adjoints, with a banded approximation, or dropping it entirely in purely local rules~\citep{williams1990efficient, bellec2020solution}. Rules such as STDP would instead replace the operator with a fixed exponentially decaying stationary (block Toeplitz) operator~\citep{bi1998synaptic}. Finally, random variants such as feedback alignment can be seen as replacing \(\P^*\) with randomness, or specific components, e.g., the occurrences of \(W\)~\citep{lillicrap2016random}. From this, one can naturally quantify the error or alignment between rules by computing the operator Frobenius norm or cosine similarity between two rules, although comparing alternative learning rules is outside the scope of the current work. This generalizes existing analysis, which focused on the infinite-width NTK for simple, wide MLP networks~\citep{bordelon2023influence}.}

A natural next step is to push this local theory further across training, to understand how these iteration-wise bottlenecks accumulate into global representation formation and persistent low-rank structure. This likely will require a stronger theory of the resolvent \(\P\), perhaps by studying the inverse of the NTK, which avoids explicitly constructing \(\P\). {Under additional simplifying assumptions, continuous-time training dynamics and the reachable domain of GD solutions may become tractable.} For example, by truncating the Neumann series for \(\P\) in Corollary~\ref{corr:recntk} to a few terms in regimes of weak recurrence \citep{Trousdale2012ImpactSpikeTimeCorrelations, Pernice2011HowStructureDeterminesCorrelations}.

{Because \(\NTK\) acts on a multi-way tensor domain, the spatial and temporal ranks used here represent only one choice of reduction, based on a space--time bipartition of the operator. Motivated by related work on low-rank recurrent updates and implicit tensor-factorization bias~\citep{pellegrino2023low,razin2022hierarchical}, future work could instead use Tucker/CP-type or other structured tensor decompositions to characterize richer low-rank motifs in the learning dynamics.} Another important direction is to study how assumptions beyond rank constraints the structure of \(\NTK\). For example, if the network dynamics collapse to fixed points or pass through dynamical bifurcations, understanding the downstream effect on \(\NTK\) could shed light on attractor formation hypotheses \citep{driscoll2024flexible, farrell2023from, qian2025alternative}. These in turn are at the core of influential ideas of computation through dynamics and dynamical motifs \citep{Vyas2020ComputationThroughNeuralPopulationDynamics}. It would be also useful to determine how structured perturbations or specialized synaptic plasticity rules, such as three-factor rules or other truncations of backpropagation, modify \(\NTK\) \citep{Lillicrap2016RandomFeedbackAlignment, Menick2020SparseApproximationRTRL}. Another interesting direction is to understand the impact of learning under \(\P\P^*\) without the operator \(\K\), matching a parameter-free, functional gradient update, or through modified forms of \(\K\) that would follow from freezing a subset of the parameters or adding synaptic plasticity. In total, our framework opens a route toward studying how stability, latent structure, bifurcations, architectural constraints, parametrization, and {the dimension and geometry of inputs shape learnability} in a wide range of practically used models.

\paragraph{Code Availability} Our code is available at \url{https://github.com/meeree/gentk}. 

\paragraph{Acknowledgments} 
We thank Saba Heravi, Tailai Li, Hanson Mo, Blake Bordelon, Janne Lappalainen, Alex Johnson Vasquez, Alex Hsu,  Zihan Zhang, Victor Shan and Mike Wendels, as well as the Computation Neuroscience Center at UW and the Allen Institute community broadly, for support, motivation, discussion and feedback.
This material is based upon work supported by the U.S. National
Science Foundation Graduate Research Fellowship Program (GRFP) under
Grant No.~DGE-2140004. Any opinions, findings, and conclusions
or recommendations expressed in this material are those of
the authors and do not necessarily reflect the views of the
National Science Foundation. 
Support was also provided under NIH grants BRAIN U01 UF1NS126485
and BRAIN R01 1RF1DA055669.
Finally, we thank the Allen Institute founder, Paul G. Allen, for his vision and support.

\section{References}

\bibliographystyle{plainnat}
\bibliography{refs}

\newpage
\appendix

\section{Appendix}

\subsection{The \href{https://github.com/meeree/gentk}{gentk} Package for Analysis of the Empirical Global NTK}
\label{sec:aprnla}

Here we detail the backend implementation of the operators used throughout the paper. Some details, including timing and implementation aspects of the randomized trace estimation routines described below, were investigated before in prior work~\citep{hazelden2025fastneuraltangentkernel}.

\subsubsection{High-Level Structure}

\paragraph{Matrix-free interface.}
Every operator in the paper is implemented as a derived class of an \texttt{Operator} base class. Importantly, all operators are implemented in a matrix-free way. Concretely, an \texttt{Operator} exposes only two methods implemented by the derived class: \texttt{matvec} and \texttt{rmatvec}. For an operator \(\mathcal{G}: S_1 \to S_2\), these correspond to application of \(\mathcal{G}\) and \(\mathcal{G}^*\), respectively, to candidate vectors in \(S_1\) or \(S_2\). Thus all analysis of the operators is carried out by probing them with input vectors \(v \in S\), rather than explicitly forming their full matrices.

{\paragraph{Construction from explicit Markovian models.} The \texttt{gentk} model interface automatically implements the construction described in Section~\ref{subsec:recglobal} by converting a generic explicit Markovian update into an implicit global-state constraint. Internal variables are stacked into the global state, update equations define the residual constraints, and automatic differentiation supplies the state-to-state and parameter-to-state Jacobian actions. For an acyclic graph, a topological ordering determines this stacking; as emphasized in the main text, that ordering is a convenient choice rather than a unique representation. Weight-based parameter edges additionally identify the forward-pass variables that are stacked as weight-sites for Theorem~\ref{thm:core}. The resulting actions of \(\P\), \(\P^*\), and \(\K\) are exposed through Jacobian-vector and vector-Jacobian products; for acyclic models the finite Neumann/path expansion can be applied directly, while more general implicit cases can use matrix-free linear solves. This construction avoids explicitly forming either the constraint Jacobians or the full \(\NTK\).}

Fortunately, there is a substantial literature on matrix-free numerical linear algebra for exactly this setting. Many such methods are randomized, estimating operator-level quantities by sketching or probing the action of the matrix on random vectors. Using these ideas, the base \texttt{Operator} class supports a range of analysis routines, including estimation of the top singular modes, alignment, effective rank, and operator norms. In addition, we expose an interface that maps closely to the theory, allowing one to compose operators, compute tensor products, compare against exact matrix representations when tractable, and form reduced partial-trace views, as described below.

\paragraph{Estimating the trace through randomized algorithms.}
One of the basic quantities exposed for every operator is its trace, i.e.\ the sum of its eigenvalues (equivalently, diagonal entries when the operator is represented as a matrix). To estimate this efficiently, we use Hutch++, which can provide accurate trace estimates with relatively few probes, especially when the operator is approximately low-rank \citep{hazelden2025fastneuraltangentkernel,meyer2021hutchoptimalstochastictrace}. If the operator is not low-rank, performance degrades, but, as discussed in the main paper, the NTK is often strongly low-rank in the settings studied here, so Hutch++ is well suited to this use case.

\paragraph{Trace-derived quantities.}
The trace is also used as a basic building block for other analysis metrics. In particular, it yields approximations to an operator's effective rank, Frobenius norm, and cosine similarity with another operator on the same domain.

\paragraph{Randomized SVD.}
One of the most informative descriptors of an operator is its leading singular structure. Matrix-free methods for estimating the top singular components are well developed~\citep{9550979}. In our implementation, we rely on \texttt{scipy}'s \texttt{LinearOperator} interface together with sparse eigensolvers such as \texttt{eigsh} for symmetric operators~\citep{2020scipy,lehoucq1998arpack}.

\subsubsection{Tensor Operator Interface}

\paragraph{Composing operators.}
Two operators in \texttt{gentk} can be combined in multiple ways. Given operators \(\mathcal{G}_1: S_1 \to S_2\) and \(\mathcal{G}_2: S_2 \to S_3\), their composition \(\mathcal{G}_2 \circ \mathcal{G}_1: S_1 \to S_3\) corresponds to the product of their explicit matrices. In our package, the syntax to compose two operators \texttt{op1} and \texttt{op2} is \texttt{op1 @ op2}. Another way to combine operators is by the tensor product \(\mathcal{G}_1 \otimes \mathcal{G}_2\), for arbitrary \(\mathcal{G}_1: S_1 \to S_2\) and \(\mathcal{G}_2: S_3 \to S_4\). The resulting operator \(\mathcal{G}_1 \otimes \mathcal{G}_2: S_1 \otimes S_3 \to S_2 \otimes S_4\) acts on tensor-product inputs. In code, this is exposed by the syntax \texttt{op1.tprod(op2)}.

\paragraph{Example syntax.}
Basic operators such as the identity and wrappers for dense numpy matrices are also implemented. For example, if \texttt{kop} denotes the operator \(\mathcal{K}\) from Proposition~\ref{prop:ntk}, and \(V\) is the \(k \times m\) weight-site matrix from Theorem~\ref{thm:core}, then the theorem can be checked numerically by verifying that the relative operator error is small:
\begin{center}
\begin{align}
    \label{eqn:verifythm}
    \texttt{err = kop.compare(MatrixWrapper(V @ V.T).tprod\_like(IdentityOperator(n), kop))}
\end{align}
\end{center}
Here \texttt{compare} computes a relative error between two operators. The wrapper \texttt{MatrixWrapper} converts the dense matrix \(V V^\top\) into an operator. The function \texttt{tprod\_like} forms the tensor product with the supplied operator and reshapes the result to match the domain and codomain structure of the reference operator, here \texttt{kop}. Finally, \texttt{IdentityOperator(n)} constructs the \(n \times n\) identity operator. In total, this uses randomized Frobenius norm estimation under the hood to verify numerically that
\[
\mathcal{K} \approx V V^\top \otimes I_n.
\]

\subsubsection{Partial Trace Views}

\begin{figure}[t]
    \centering
    \includegraphics[width=.99\linewidth]{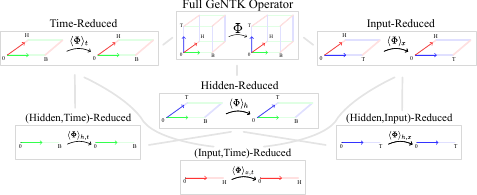}
    \caption{\textbf{Examples of partial reductions of an operator acting on a 3-tensor domain, \(\R^{B \times T \times H}\).} The operator \(\Phi\) acts on a domain \(\R^{B \times T \times H}\), representing batch, time, and hidden-unit axes. Each reduction averages over one or more axes. The reduced operators at the bottom are simple \(B \times B\), \(T \times T\), and \(H \times H\) matrices, respectively. These capture how the operator varies across batch inputs, hidden units, or timesteps after averaging over the complementary axes.}
    \label{fig:reduce}
\end{figure}

\paragraph{The partial trace.}
The operators discussed in the paper act naturally on tensor domains. It is common to vectorize such domains and ignore this structure, but keeping the tensor structure explicit yields more informative reduced views. For example, the recurrent NTK operator from Corollary~\ref{corr:recntk} acts on 3-tensors in \(\R^{n_x \times n_t \times n}\). Equivalently, the operator itself can be viewed as a 6-tensor. A useful tool for analyzing such an object is the \emph{partial trace}, which traces over one or more tensor axes.

As in the main text, we define reduced operators such as \(\texttt{ntk}_{\mathrm{temp}}\), which averages over the hidden dimension \(n\) and yields an operator acting on the \(n_x \times n_t\) axes, and \(\texttt{ntk}_{\mathrm{space}}\), which averages over the batch and time axes and yields an \(n \times n\) matrix. The effective ranks of these reduced operators are used in Figure~\ref{fig:regions} of the main text.

In code, the partial trace is implemented, for example, by
\[
\texttt{ntk\_temp = ntk.partial\_avg(2)},
\]
which averages over the final axis. Under the hood, this is done in a matrix-free way: (1) inputs to the reduced operator are expanded across the traced axis, (2) passed through the original operator, and (3) the outputs are averaged over that same axis. Concretely, for \(x \in \R^{n_x \cdot n_t}\),
\[
\NTK^{\mathrm{temp}}(x) = \frac{1}{n}\sum_{j=1}^n \NTK(x e_j^\top),
\]
where the matrix \(x e_j^\top\) corresponds to expanding \(x\) into the missing hidden dimension.

These reduced operators can then be combined with SVD routines to compute reduced singular modes, for example dominant time-by-trial modes after averaging over hidden units, as in Figure~\ref{fig:skree} in the main text. For an operator acting on \(\R^{B \times T \times H}\), there are \(2^3-1=7\) nontrivial partial reductions, giving a family of complementary views of how the operator acts across different axes.

\paragraph{Forming explicit reduced matrices.}
Finally, once an operator has been reduced enough, it can become tractable to form explicitly. For example, \(\texttt{ntk}_{\mathrm{space}}\) is an \(n \times n\) matrix, where \(n\) is the hidden dimension, and so can often be constructed exactly. To do this, we provide a method \texttt{op.full\_matrix()} that forms the corresponding dense matrix. For instance, the batch-by-batch reduced NTK can be computed via
\begin{center}
    \texttt{ntk\_batches = ntk.partial\_avg((1,2)).full\_matrix()}
\end{center}
which averages over time and hidden dimensions before explicitly forming the reduced operator. Thus, although treating the NTK as a tensor operator is conceptually more involved, it naturally yields multiple informative reductions tailored to different questions.

\subsection{Experimental Appendix}
\label{subsec:expap}

Here we give additional details for the experiments in the main text.

\subsubsection{GRU Memory Pro}
\label{subsec:gruexap}

\paragraph{Constructing the 5-fixed-point initialization (Network 2).}
We use \(n=256\) neurons in these experiments. The goal is to construct a recurrent initialization with multiple nontrivial fixed points at initialization, and then use the same idea to initialize the GRU recurrent weights.

\begin{figure}[t]
    \centering
    \includegraphics[width=\linewidth]{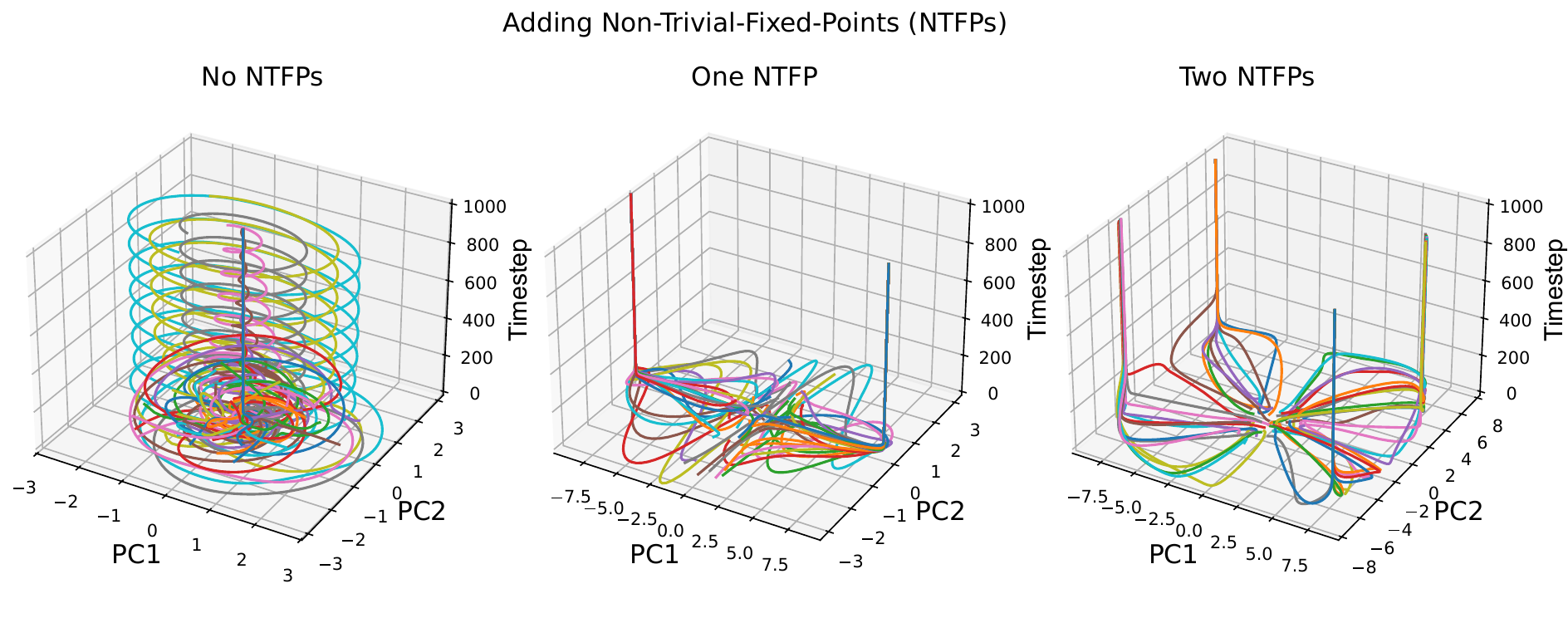}
    \caption{\textbf{Dynamics of an RNN with added non-trivial fixed points (NTFPs).} Random weights are sampled with Xavier initialization, scaled by gain \(g\), with \(n=256\) and \(g\) varied between \(1\) and \(2\). Non-trivial fixed points are then added to the model as described in the text. Plotted trajectories correspond to distinct random initial conditions and gain values. The random seed is fixed between trials and the dynamics are projected into a shared PCA space. Time is shown on the vertical axis. In the absence of added NTFPs, the RNN transitions between collapse and chaotic behavior. For small gain, trajectories collapse to the trivial fixed point at zero; for larger gain, the dynamics become unstable. With one added NTFP, the system becomes effectively bistable under all gains shown. With two added NTFPs, the dynamics collapse to four distinct fixed points. The number of effective fixed points doubles because \(\tanh\) is odd.}
    \label{fig:addingntfp}
\end{figure}

Specifically, consider an RNN in the response period of the Memory-Pro task, where the task input is zero. We want to construct an RNN with nontrivial fixed points, rather than having all task-driven trajectories collapse to the trivial fixed point, as happens under a default Xavier initialization. Consider the update
\[
h_{t+1} = W \sigma(h_t).
\]
Then
\[
h_{t+1} - h_t = W \sigma(h_t) - h_t.
\]
Setting this to zero at a desired fixed point \(\bar h\) gives
\[
W \sigma(\bar h) - \bar h = 0.
\]
This is achieved by any \(W\) of the form
\[
W = \frac{\bar h_1 \sigma(\bar h_1)^\top}{\|\sigma(\bar h_1)\|^2} + W_1,
\]
where \(W_1\) is any matrix that maps \(\sigma(\bar h_1)\) to zero.

More generally, if \(\sigma(\bar h_1),\ldots,\sigma(\bar h_m)\) are orthogonal, then we can construct a matrix \(W\) with all of these points fixed:
\begin{align}
    W = \sum_{i=1}^m \frac{\bar h_i \sigma(\bar h_i)^\top}{\|\sigma(\bar h_i)\|^2} + W_{1:m},
\end{align}
where \(W_{1:m}\) annihilates each \(\sigma(\bar h_i)\). Indeed,
\begin{align}
    W \sigma(\bar h_j)
    = \sum_{i=1}^m \frac{\bar h_i \, \delta_{ij}\,\|\sigma(\bar h_i)\|^2}{\|\sigma(\bar h_i)\|^2} + 0
    = \bar h_j.
\end{align}
A natural choice is to take the \(\bar h_i\) to be scaled orthogonal coordinate directions, since \(\sigma(k e_j)=\sigma(k)e_j\) for coordinate vectors when \(\sigma\) acts componentwise. In that case the orthogonality of the transformed vectors is immediate.

\paragraph{Task details.}
We use fixed-duration stimulus, memory, and response periods so that the hidden-state and output dynamics are easier to interpret. Each trial has \(n_t=90\) timesteps, with \(30\) timesteps per task phase. We use batch size \(500\), with \(5{,}000\) total task inputs and targets. During training, Gaussian noise with variance \(3.2\) is added to the inputs to encourage more robust attractor structure; this noise is turned off during evaluation and plotting \citep{driscoll2024flexible,qian2025alternative}. The learning rate is \(10^{-3}\).

\begin{figure}[t]
    \centering
    \includegraphics[width=\linewidth]{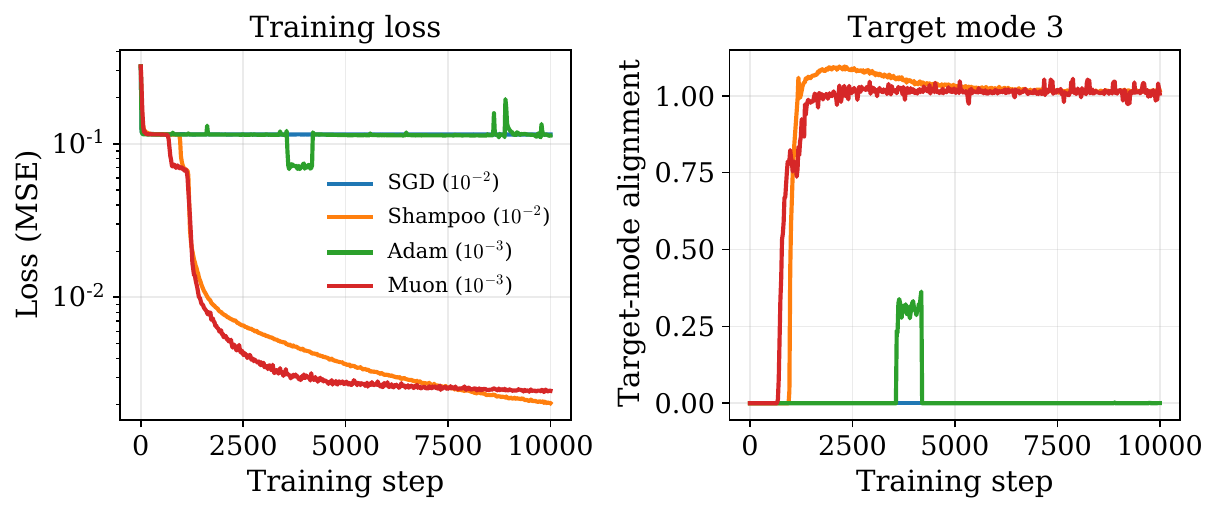}
    \caption{{ \textbf{Additional optimizer comparison in the context of Figure~\ref{fig:memproinit}, training GRU on the MemoryPro task~\citep{driscoll2024flexible}}. Here, we compare SGD, Adam, Shampoo and Muon for the Network 1 configuration with default Xavier initialization~\citep{kingma2014adam, gupta2018shampoopreconditionedstochastictensor, jordan2024muon}. \textbf{A} loss trajectories on the task. \textbf{B} corresponding alignment of $\NTK$ with the task targets. High alignment is associated with lower loss.} }
    \label{fig:optimizermatched}
\end{figure}

{ \paragraph{Comparing additional optimizers.} 
Figure~\ref{fig:optimizermatched} extends the optimizer comparison in the
main text to Adam and Muon~\citep{kingma2014adam,jordan2024muon}. Muon and
Shampoo both escape the plateau observed with SGD, while Adam does not in
this example, likely reflecting the fact that Adam relies on a very approximate reconditioning of the gradient updates, assuming a diagonal structure of the inverse Fisher information. Successful optimization is accompanied by increased
$\NTK$-target alignment. Thus, the failure of SGD is not simply due to the
absence of a parameter-space solution: changing the update geometry can
substantially alter access to the poorly aligned task modes. 
One can rely on either (1) a more accurate optimizer, such as Muon and Shampoo here, potentially overcoming the self-referential bias in the main text, or (2) initializing the model in a way that has higher $\NTK$-target alignment. 
Doing (2) in general in a principled way using the operators introduced here is a potential area of future research.
}

\subsubsection{Vanilla RNN Space-Time GeNTK Rank}
\label{subsec:rnnexap}

\paragraph{Task setup.}
For the spatial and temporal regime experiments in Section~5.3, we use a nonlinear student--teacher RNN setting. The teacher and student share the same architecture and are both driven by i.i.d.\ standard normal input sequences. The teacher is initialized with a fixed recurrent gain \(g^*\), which is set to \(g^*=1\) in all experiments. The student uses the same architecture, but one subset of weights is re-initialized.

In the recurrent-weight task, the student's input and output weights are set equal to the teacher's, while the recurrent weights \(W\) are re-sampled with Xavier initialization and gain \(g\). The goal is then to recover the teacher output by training only \(W\). In the input-weight task, the student's recurrent and output weights are instead matched to the teacher, while the input weights \(W_{\mathrm{in}}\) are re-sampled with Xavier initialization and gain \(g\), and only these weights are trained. In both cases the task is realizable by returning the perturbed weights to the teacher values, so slow or incomplete learning reflects bias in the induced NTK rather than task misspecification.

Unless otherwise stated, we use \(n_{\mathrm{in}}=16\) input channels, \(n=64\) hidden units, \(n_{\mathrm{out}}=1\) output channel, and \(n_t=40\) timesteps, with an MSE loss applied at every timestep. Models are trained with vanilla SGD using batch size \(128\), learning rate \(0.01\), and no momentum or gradient clipping.

{
\subsubsection{Width-scaling controls}
\label{subsec:rnnwidthcontrols}

For Figure~\ref{fig:rnnwidth}, we use the RNN setup of Section~\ref{subsec:regimes} with both recurrent and input weights trainable, \(g=1\), three random seeds, and \(21\) Hutchinson probes. We additionally fix width to \(1024\) and vary sequence length, evaluation batch size, and input rank (Figure~\ref{fig:rnnwidthcontrols}). Temporal rank changes little with batch size, grows sublinearly with sequence length, and does not track the strong increase in joint hidden--input activity rank as input rank increases. Thus, the large-width temporal plateau is not explained by a trivial sampling or activity-rank ceiling.
}

\begin{figure}[t]
    \centering
    \includegraphics[width=.98\linewidth]{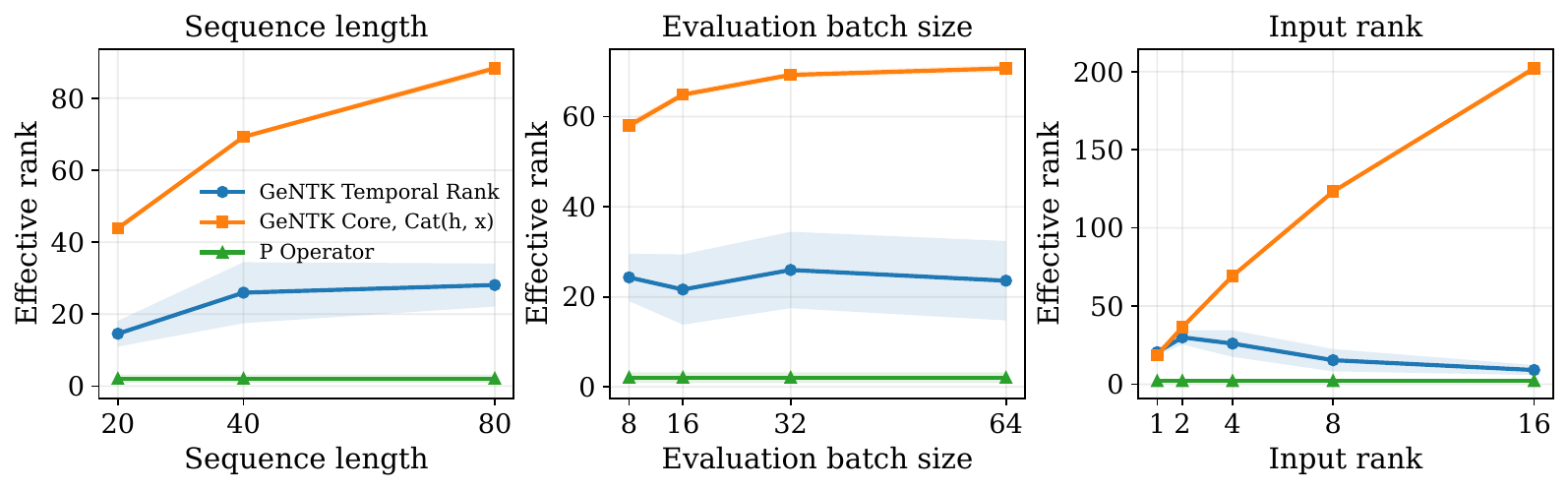}
    \caption{{\textbf{Controls for the temporal-rank plateau in the RNN width experiment.} At fixed hidden width \(n=1024\), we independently vary sequence length, evaluation batch size, and input rank. Temporal NTK effective rank is nearly unchanged by an eightfold increase in batch size, grows sublinearly with sequence length, and does not follow the strong increase in joint hidden--input activity rank with input rank. The resolvent effective rank is shown as an additional reference.}}
    \label{fig:rnnwidthcontrols}
\end{figure}

\subsubsection{Transformer GeNTK}
\label{subsec:transexap}

\begin{figure}[t]
    \centering
    \includegraphics[width=.99\linewidth]{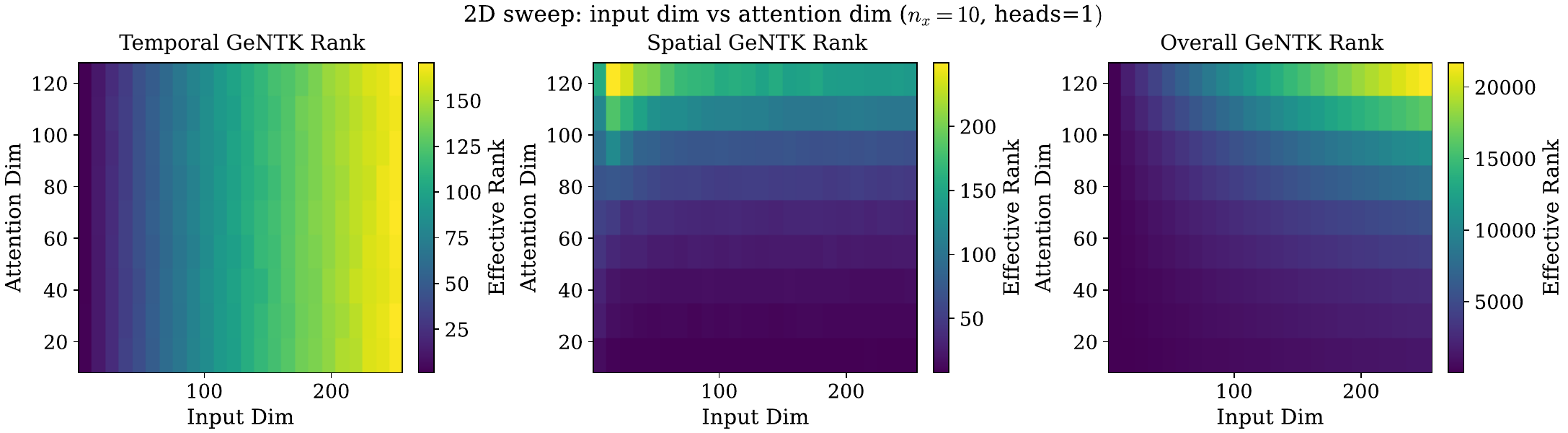}
    \caption{\textbf{Temporal, spatial, and overall NTK effective rank for a two-dimensional sweep over input dimension and attention dimension.} The temporal rank grows primarily with input dimension, while the spatial and overall rank increase strongly with attention width.}
    \label{fig:twodaptrans}
\end{figure}

\paragraph{Transformer setup.}
For the transformer experiment, we use a single self-attention block followed by an MLP, evaluated at initialization on i.i.d.\ Gaussian input sequences \(X \in \R^{n_x \times n_t \times n_{\mathrm{in}}}\). We vary \(n_x \in \{5,10,15\}\), use attention width \(n_{\mathrm{attn}}=64\), \(n_{\mathrm{out}}=1\), and \(n_t=50\), and take the MLP to have \(3\) layers with \(128\) neurons per layer. We then vary the input dimension \(n_{\mathrm{in}}\) and measure the temporal rank of the resulting global-state NTK as a function of input dimension and batch size. The relevant weight-site matrix is computed from the transformer and MLP weight-sites, consisting of the projected input, attention output, and hidden activations of the MLP (excluding the final readout layer), as derived in Appendix~\ref{subsec:corederivap}.

To test whether this bottleneck can be relieved, we also apply a deterministic one-dimensional Fourier-feature embedding~\citep{tancik2020fourierfeaturesletnetworks} before the input projection and sweep the number of Fourier frequencies. As the feature dimension increases, both the effective rank of the embedded input and the temporal rank of \(\NTK\) increase.

\begin{figure}[t]
    \centering
    \includegraphics[width=.99\linewidth]{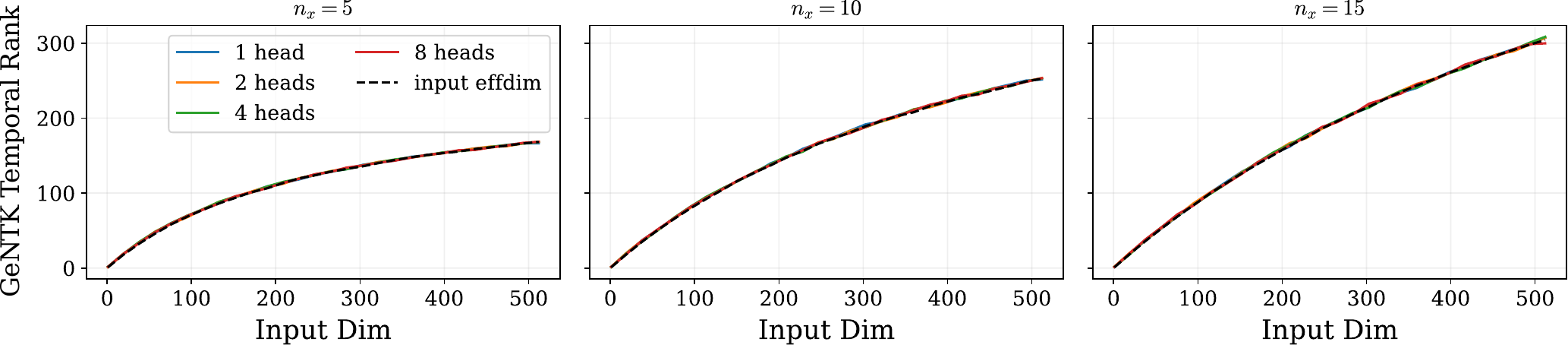}
    \caption{\textbf{Temporal NTK rank versus input dimension for a single-block self-attention toy model with varying head count.} Increasing the number of attention heads has only a modest effect on temporal rank, so the same basic temporal bottleneck from the input representation remains.}
    \label{fig:apmultihead}
\end{figure}

Additional appendix sweeps support the same interpretation. Varying the number of attention heads changes the temporal rank only modestly across input dimensions, so multi-head attention does not remove the same basic temporal bottleneck in this toy setting (Figure~\ref{fig:apmultihead}). By contrast, a two-dimensional sweep over input dimension and attention width shows a clearer separation: the temporal rank grows primarily with input dimension, whereas the spatial and overall ranks grow strongly with attention width (Figure~\ref{fig:twodaptrans}). Thus, widening attention mainly enriches the spatial and overall operator structure, while the temporal expressivity of \(\NTK\) remains bottlenecked by the input representation.

{
\subsection{Validation on Pretrained Transformer Language Models}
\label{subsec:pythia}

To demonstrate that the proposed decomposition is not restricted to small recurrent or attention models, we apply it to pretrained Pythia transformer language models~\citep{biderman2023pythia}. We consider models ranging from 70M to 410M parameters and compute the global-state operators $\P$, $\K$, and $\NTK$ matrix-free. The global state contains the transformer block states; the final normalization and language-model head are treated as downstream readout and excluded from the hidden-state parameter operator. As shown in Figure~\ref{fig:pythia}A, independently differentiating the hidden-state NTK agrees with the factorization $\NTK=\P\K\P^*$ to relative error below $10^{-4}$ across all three model sizes. Likewise, the explicit weight-site construction accurately reproduces $\K=(D_\theta\F)(D_\theta\F)^*$ (Figure~\ref{fig:pythia}B).

The calculation is matrix-free: the reference hidden-state NTK is differentiated independently using Jacobian--vector and vector--Jacobian products, while the finite-depth transformer constraint is triangular so that $\P=(D_h\F)^{-1}$ can be applied by exact layerwise triangular sweeps rather than by explicitly forming the inverse. Thus, the same factorization calculations can be carried out directly on a pretrained transformer with hundreds of millions of parameters.

\begin{figure}[t]
    \centering
    \includegraphics[width=.7\linewidth]{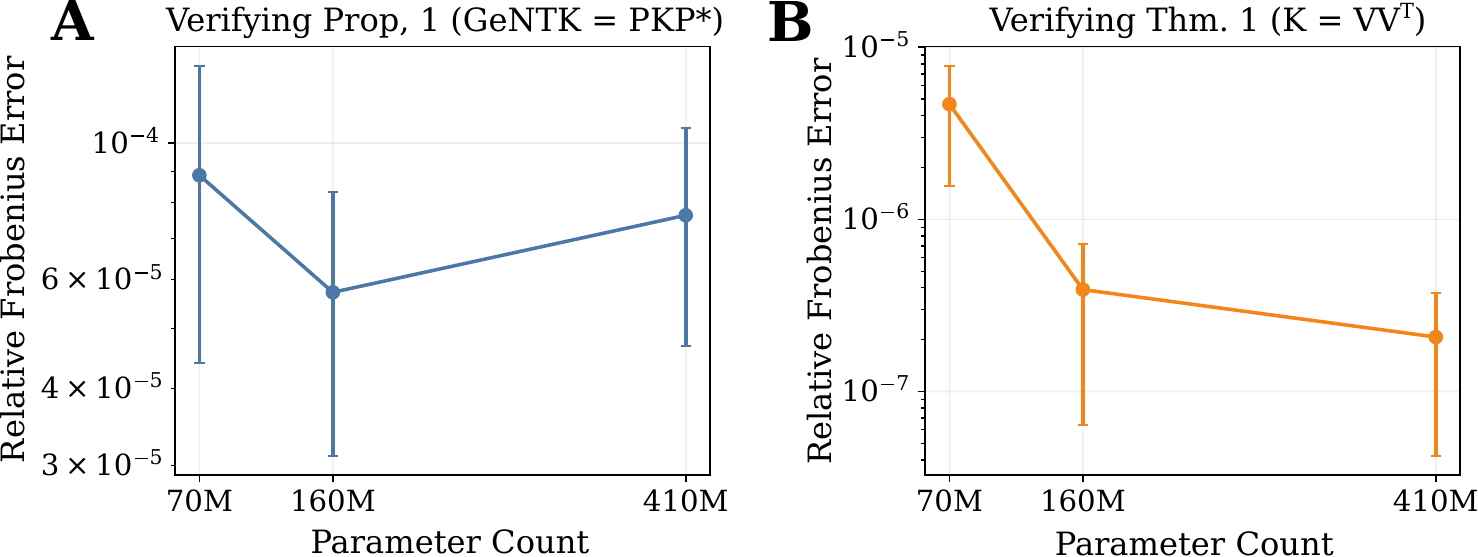}
    \caption{{\textbf{Validation of the factorization and weight-site core on pretrained Pythia transformer language models~\citep{biderman2023pythia}.} \textbf{A} Matrix-free estimate of relative Frobenius error \(\|\NTK - \P\K\P^*\|_F / \|\NTK\|_F\), comparing the hidden-state NTK differentiated directly with Jacobian--vector/vector--Jacobian products against the independently factored operator \(\P\K\P^*\), for Pythia 70M, 160M, and 410M. \textbf{B} Corresponding matrix-free relative error between \(\K=(D_\theta\F)(D_\theta\F)^*\) and the explicit weight-site core \(VV^*\) (with the corresponding identity factor and axis bookkeeping implied by Theorem~\ref{thm:core}).}}
    \label{fig:pythia}
\end{figure}
}

\subsection{Proof of Theorem~\ref{thm:core}}
\label{proof:core}

\subsubsection{Weight-Based Model Definition}
\label{def:weighttied}

A weight-based model is one in which the parameters \(\theta\) can be represented, after lifting if necessary, by a global matrix \(W \in \R^{n \times m}\) that is applied to internal \emph{weight-site} variables collected in a matrix \(V \in \R^{k \times m}\). Letting
\[
Q := V W^\top \in \R^{k \times n},
\]
we assume the model can be written in the form
\[
\mathcal{F}(h,x,\theta)=\mathcal{G}_h(h,V,Q,x),
\]
where \(V,Q,W\) are uniquely determined by a lifted system of constraints in the augmented state space \((h,V,Q)\). Concretely, the lifted system has the form
\begin{align}
    \label{eqn:augf}
    0=\mathcal{F}_{\mathrm{aug}}((h,V,Q),x,W)
    =
    \begin{pmatrix}
        \mathcal{G}_h(h,V,Q,x) \\
        \mathcal{G}_V(h,V,Q,x) \\
        Q - V W^\top
    \end{pmatrix},
\end{align}
where \(\mathcal{G}_V(h,V,Q,x)=0\) determines the weight-site variables \(V\), and the final constraint enforces \(Q=VW^\top\). We assume \(\mathcal{G}_h\) maps into \(S\), and \(\mathcal{G}_V\) maps into \(\R^{k \times m}\).

This definition guarantees that the parameter dependence enters through matrix--vector interactions with the weight sites. Allowing \(V\) to be defined implicitly rather than explicitly in terms of \(h\) and \(x\) accommodates nested architectures, where intermediate weight-site variables themselves depend on earlier weighted computations.

\paragraph{Infinite weight-site case.}
More generally, one may allow infinitely many weight sites (\(k=\infty\)), but we avoid that notation here for simplicity. The proof is unchanged in essence: the object \(VV^\top\) is then interpreted as a second-moment operator constructed by probing, rather than as an explicitly formed matrix. Its rank is still bounded by the feature dimension \(m\), so a finite number of probes suffices to identify its range.

\paragraph{Motivating nested example.}
Consider the nested map \(y = W_2 \sigma(W_1 x)\). Define
\[
v_1 := x,\qquad q_1 := W_1 v_1,\qquad v_2 := \sigma(q_1),\qquad q_2 := W_2 v_2,
\]
so that \(y=q_2\). Stacking the weights into a block-diagonal matrix \(W=\mathrm{blockdiag}(W_1,W_2)\), the weight sites can be written as
\[
V =
\begin{pmatrix}
x & 0 \\
0 & \sigma(q_1)
\end{pmatrix},
\qquad
Q = V W^\top.
\]
This can be formulated as a lifted system of constraints in \((y,V,Q)\). The point is that even nested architectures can be rewritten so that all parameter dependence appears through products of a shared weight matrix with weight-site variables.

\subsubsection{Proof}

\begin{proof}
Lift the model into the augmented state variables \(W \in \R^{n \times m}\), \(V \in \R^{k \times m}\), and \(Q \in \R^{k \times n}\) using the constraint system in Equation~\eqref{eqn:augf}:
\[
0 = \mathcal{F}_{\mathrm{aug}}((h,V,Q),x,W)
=
\begin{pmatrix}
\mathcal{G}_h(h,V,Q,x) \\
\mathcal{G}_V(h,V,Q,x) \\
Q - V W^\top
\end{pmatrix}.
\]
By construction, the original constraint satisfies
\[
\mathcal{F}(h,x,\theta)=\mathcal{G}_h(h,V,Q,x),
\]
with \(V\) and \(Q\) uniquely determined by the lifted system at the current state.

Define the augmented state space
\[
S_{\mathrm{aug}} := S \oplus \R^{k \times m} \oplus \R^{k \times n},
\qquad
h_{\mathrm{aug}} := (h,V,Q) \in S_{\mathrm{aug}}.
\]
Applying Proposition~\ref{prop:ntk} to the lifted system yields
\[
\mathcal{P}_{\mathrm{aug}} := (D_{h_{\mathrm{aug}}}\mathcal{F}_{\mathrm{aug}})^{-1},
\qquad
\mathcal{K}_{\mathrm{aug}} := (D_W \mathcal{F}_{\mathrm{aug}})(D_W \mathcal{F}_{\mathrm{aug}})^*.
\]
The corresponding first-order correction in the lifted state is
\[
\delta h_{\mathrm{aug}}
=
\begin{pmatrix}
\delta h\\
\delta V\\
\delta Q
\end{pmatrix}
=
\text{GeNTK}_{\mathrm{aug}}
\begin{pmatrix}
\err(h)\\
0\\
0
\end{pmatrix}.
\]
If we view \(\mathcal{P}_{\mathrm{aug}}\) as a \(3\times 3\) block operator, then the correction in the original state \(h\) is obtained from the first row:
\[
\delta h
=
\mathcal{P}_{\mathrm{aug},1}\,
\mathcal{K}_{\mathrm{aug}}\,
\mathcal{P}_{\mathrm{aug},1}^*(\err(h)).
\]
Since \(\err(h)\) is arbitrary, this shows that the original NTK can be written as
\[
\NTK
=
\mathcal{P}_{\mathrm{aug},1}\,
\mathcal{K}_{\mathrm{aug}}\,
\mathcal{P}_{\mathrm{aug},1}^*.
\]

Next, we compute \(\mathcal{K}_{\mathrm{aug}}\). Since only the final block \(Q - V W^\top\) depends on \(W\),
\[
D_W \mathcal{F}_{\mathrm{aug}}
=
\begin{pmatrix}
0\\
0\\
-(V \otimes I_n)
\end{pmatrix}.
\]
Therefore
\[
\mathcal{K}_{\mathrm{aug}}
=
\begin{pmatrix}
0 & 0 & 0\\
0 & 0 & 0\\
0 & 0 & VV^\top \otimes I_n
\end{pmatrix}.
\]
Hence only the \((1,3)\) block of \(\mathcal{P}_{\mathrm{aug}}\) contributes to the induced NTK on \(S\), giving
\[
\NTK
=
\mathcal{P}_{\mathrm{aug},1,3}\,
(VV^\top \otimes I_n)\,
\mathcal{P}_{\mathrm{aug},1,3}^*.
\]
This is the desired factorization. The matrix \(VV^\top\) is the Gram matrix of the weight-site state, so it acts as an unnormalized projector onto the dominant modes already present in the weight sites. This is the Kronecker-core bottleneck described in the theorem.

\paragraph{Existence of \(\mathcal{P}_{\mathrm{aug},1,3}\).}
For completeness, we justify that the block \(\mathcal{P}_{\mathrm{aug},1,3}\) is well defined. Explicitly,
\[
D_{h_{\mathrm{aug}}}\mathcal{F}_{\mathrm{aug}}
=
\begin{pmatrix}
D_h \mathcal{G}_h & D_V \mathcal{G}_h & D_Q \mathcal{G}_h\\
D_h \mathcal{G}_V & D_V \mathcal{G}_V & D_Q \mathcal{G}_V\\
0 & -(W \otimes I_k) & I_{kn}
\end{pmatrix}.
\]
{By the local well-posedness assumption in Theorem~\ref{thm:core}, the full lifted Jacobian above is invertible, so \(\mathcal{P}_{\mathrm{aug}}\) and in particular the block \(\mathcal{P}_{\mathrm{aug},1,3}\) are well defined. When the indicated \(V,Q\) subblock is invertible, the same block inverse can be written using the corresponding Schur complement}
\[
C
:=
D_h \mathcal{G}_h
-
\bigl[D_V \mathcal{G}_h + D_Q \mathcal{G}_h (W \otimes I_k)\bigr]
\bigl[D_V \mathcal{G}_V + D_Q \mathcal{G}_V (W \otimes I_k)\bigr]^{-1}
D_h \mathcal{G}_V.
\]
{By construction \(\mathcal{F}(h,x,\theta)=\mathcal{G}_h(h,V,Q,x)\). Using \(Q=VW^\top\), \(\mathcal{G}_V(h,V,Q,x)=0\), and the implicit function theorem for the intermediate variables gives \(C=D_h\mathcal{F}\). Thus this Schur-complement form recovers the original propagator inside the lifted block inversion.}

Writing the corresponding block inverse explicitly gives
\begin{align}
\label{eqn:newp}
\mathcal{P}_{\mathrm{aug},1,3}
=
\mathcal{P}\,
\Bigl(
D_h \mathcal{F}
+
\bigl[D_V \mathcal{G}_h + D_Q \mathcal{G}_h (W \otimes I_k)\bigr]
\bigl[D_V \mathcal{G}_V + D_Q \mathcal{G}_V (W \otimes I_k)\bigr]^{-1}
D_Q \mathcal{G}_V
\Bigr).
\end{align}
The precise form is not especially important in applications; what matters is that it exists and induces the factorization
\[
\NTK = \mathcal{P}_{\mathrm{core}} (VV^\top \otimes I_n)\mathcal{P}_{\mathrm{core}}^*,
\qquad
\mathcal{P}_{\mathrm{core}} := \mathcal{P}_{\mathrm{aug},1,3}.
\]
In practice, one usually identifies the weight sites directly from the architecture, rather than computing Equation~\eqref{eqn:newp} explicitly.
\end{proof}

\subsection{Proof of Proposition~\ref{prop:spacetime}}
\label{subsec:stproofap}

\begin{proof}
For a broader discussion, including Schmidt decompositions of operators over tensor-product spaces, see \citep{nielsen2010quantum,eckart1936approximation}. Here we only need the fact that the NTK is positive semidefinite, since it is a Gram operator.

We first show that the reduced views are also positive semidefinite. Let \((e_j)_{j=1}^n\) be an orthonormal basis of \(\R^n\), and let \(u \in \R^k\). Then
\[
\langle u, \NTK^{\mathrm{temp}}(u)\rangle
=
\frac{1}{n}\sum_{j=1}^n
\langle u e_j^\top, \NTK(u e_j^\top)\rangle_F.
\]
Each term in the sum is nonnegative because \(\NTK\) is PSD, hence \(\NTK^{\mathrm{temp}}\) is PSD. The same argument shows that \(\NTK^{\mathrm{space}}\) is PSD.

Now let \(E \in \R^{k \times n}\) be any error signal, and let \(\delta h = \NTK(E)\). We will show
\begin{align}
    \im(\NTK(E)) \subseteq \im(\NTK^{\mathrm{temp}}), \tag{F1}
\end{align}
and
\begin{align}
    \im(\NTK(E)^\top) \subseteq \im(\NTK^{\mathrm{space}}). \tag{F2}
\end{align}
These two inclusions imply
\[
\rank(\delta h)
\le
\min\!\bigl(
\rank(\NTK^{\mathrm{temp}}),
\rank(\NTK^{\mathrm{space}})
\bigr),
\]
since row and column rank are equal.

We now prove F1. Let \(u \in \ker(\NTK^{\mathrm{temp}})\). Since \(\NTK^{\mathrm{temp}}\) is PSD,
\[
0
=
\langle u,\NTK^{\mathrm{temp}}(u)\rangle
=
\frac{1}{n}\sum_{j=1}^n
\langle u e_j^\top, \NTK(u e_j^\top)\rangle_F.
\]
Each term is nonnegative, so each term must in fact be zero:
\[
\langle u e_j^\top, \NTK(u e_j^\top)\rangle_F = 0,
\qquad j=1,\dots,n.
\]
Because \(\NTK\) is PSD, this implies
\[
\NTK(u e_j^\top)=0,
\qquad j=1,\dots,n.
\]
Using symmetry of \(\NTK\), for any \(j\),
\[
\langle u, \NTK(E)e_j\rangle
=
\langle u e_j^\top, \NTK(E)\rangle_F
=
\langle \NTK(u e_j^\top), E\rangle_F
=0.
\]
Therefore \(u^\top \NTK(E)=0\). Since this holds for all \(u \in \ker(\NTK^{\mathrm{temp}})\), we conclude
\[
\im(\NTK(E)) \subseteq \ker(\NTK^{\mathrm{temp}})^\perp.
\]
Because \(\NTK^{\mathrm{temp}}\) is symmetric,
\[
\ker(\NTK^{\mathrm{temp}})^\perp = \im(\NTK^{\mathrm{temp}}),
\]
which proves F1.

The proof of F2 is identical after exchanging the temporal and spatial roles. If \(v \in \ker(\NTK^{\mathrm{space}})\), then the same argument shows
\[
\NTK(E) v = 0.
\]
Hence
\[
\im(\NTK(E)^\top) \subseteq \ker(\NTK^{\mathrm{space}})^\perp = \im(\NTK^{\mathrm{space}}).
\]
This proves F2 and completes the argument.
\end{proof}

\subsection{Deriving Weight-Sites \texorpdfstring{$V$}{V}}
\label{subsec:corederivap}

\subsubsection{RNN}
\label{subsec:rnncore}

Consider the recurrent network
\begin{align}
    h_j(t+1) &= \phi\!\bigl(W h_j(t) + W_{\mathrm{in}} x_j(t+1)\bigr), \qquad h_j(0):=h_0,
\end{align}
with output
\begin{align}
    y_j(t) &= W_{\mathrm{out}} h_j(t),
\end{align}
where \(W\), \(W_{\mathrm{in}}\), and \(W_{\mathrm{out}}\) are trainable.

The trainable weights appear only through multiplication of the current hidden state, current input, and current hidden state at readout. Thus the relevant weight-site variables are simply \(h_j(t)\) and \(x_j(t+1)\) for the recurrent and input weights, together with \(h_j(t)\) again for the output map. If one stacks all trainable matrices into a block-diagonal matrix
\[
\widetilde W = \mathrm{blockdiag}(W, W_{\mathrm{in}}, W_{\mathrm{out}}),
\]
then the corresponding weight-site matrix may be written schematically as
\[
V = \mathrm{cat}(H_-, X, H),
\]
where \(H_-\) collects the pre-update hidden states \(h_j(t)\), \(X\) collects the inputs \(x_j(t+1)\), and \(H\) collects the hidden states feeding the readout. If the readout weights are not included in the NTK under consideration, one simply omits the final block.

In the common case where only recurrent and input weights are trained, this reduces to
\[
V = \mathrm{cat}(H_-, X).
\]
Consequently,
\[
\mathcal{K} = V V^\top \otimes I_n,
\]
up to the obvious block-size bookkeeping associated with stacking the trainable weight matrices. In expanded form, this corresponds to a sum of Gram contributions from each trainable weight family:
\[
\mathcal{K}
=
(H_- H_-^\top + X X^\top)\otimes I_n
\]
when only \(W\) and \(W_{\mathrm{in}}\) are trained. Thus the recurrent global-state NTK is bottlenecked by the joint hidden--input activity over all batches and timesteps.

\subsubsection{GRU}
\label{subsec:grucorederive}

Consider the GRU without biases, for simplicity:
\begin{align}
\begin{split}
    r_j(t+1)  &= \sigma(W_{ir} x_j(t+1) + W_{hr} h_j(t)), \\
    z_j(t+1) &= \sigma(W_{iz} x_j(t+1) + W_{hz} h_j(t)), \\
    \ell_j(t+1) &= \tanh(W_{i\ell} x_j(t+1) + r_j(t+1) \odot (W_{h\ell} h_j(t))), \\
    h_j(t+1) &= (1-z_j(t+1)) \odot \ell_j(t+1) + z_j(t+1)\odot h_j(t).
\end{split}
\end{align}
Stack the weights \(W_{ir},W_{iz},W_{i\ell}\) into a single matrix \(W_{\mathrm{in}}\), and \(W_{hr},W_{hz},W_{h\ell}\) into \(W\). Then the model can be written in the lifted form
\begin{align}
\begin{split}
    \label{eqn:weightsitegru}
    q_j(t+1) &= W h_j(t), \\
    p_j(t+1) &= W_{\mathrm{in}} x_j(t+1), \\
    h_j(t+1) &= \phi_{\mathrm{step}}\!\bigl(q_j(t+1), p_j(t+1), h_j(t)\bigr),
\end{split}
\end{align}
where \(\phi_{\mathrm{step}}\) performs the remaining GRU nonlinearities and gating, but does not involve the trainable weights explicitly.

More explicitly,
\[
\phi_{\mathrm{step}}(q,p,h)
=
(1-\sigma(p_2+q_2)) \odot \tanh\!\bigl(p_3 + \sigma(p_1+q_1)\odot q_3\bigr)
+
\sigma(p_2+q_2)\odot h,
\]
where \(q_1,q_2,q_3\) extract the pieces of \(q\) corresponding to \(W_{hr},W_{hz},W_{h\ell}\), respectively, and similarly for \(p_1,p_2,p_3\).

Thus the weight-site augmented state is simply the collection of hidden states and inputs:
\[
V = \mathrm{cat}\{(h_j(t),x_j(t+1))\}_{1 \le j \le n_x,\ 0 \le t \le n_t-1}.
\]
This is already evident from the original GRU equations: the trainable weights only multiply \(h_j(t)\) and \(x_j(t+1)\), and no additional weighted intermediate variables are required.

\paragraph{Verifying the Kronecker core.}
To make the decomposition explicit, define \(v_j(t)=(q_j(t),p_j(t),h_j(t))\), with initial condition \(v_j(0)=(0,0,h_0)\), and let \(f_v\) denote the one-step lifted dynamics in Equation~\eqref{eqn:weightsitegru}. The matrix \(W\) has shape \(3n \times n\), while \(W_{\mathrm{in}}\) has shape \(3n \times n_{\mathrm{in}}\). Then
\[
D_W f_v = [h_j(t),0]^T \otimes I_{3n},
\qquad
D_{W_{\mathrm{in}}} f_v = [0,x_j(t+1)]^T \otimes I_{3n},
\]
up to the obvious reshaping induced by stacking the gates.

Let \(H_-\) denote the matrix collecting all \(h_j(t)\) for \(t=0,\dots,n_t-1\), and let \(X\) denote the matrix collecting all \(x_j(t+1)\) for \(t=0,\dots,n_t-1\). Then
\[
\mathcal{K}
=
(D_W f_v)(D_W f_v)^\top + (D_{W_{\mathrm{in}}} f_v)(D_{W_{\mathrm{in}}} f_v)^\top
=
(H_-H_-^\top + XX^\top)\otimes I_{3n}.
\]
Equivalently, concatenating \(H_-\) and \(X\) into the augmented state matrix \(V\),
\[
\mathcal{K} = V V^\top \otimes I_{3n},
\]
which is exactly the Kronecker-core factorization used in the main text.

\subsubsection{Single-Block Transformer}
\label{subsec:aptrans}

\newcommand{\wbox}[1]{\colorbox{green!50}{$\displaystyle #1$}}

We write the self-attention block in the form
\[
Q = X\,\wbox{W_Q^\top},
\qquad
K = X\,\wbox{W_K^\top},
\qquad
V_{\mathrm{val}} = X\,\wbox{W_V^\top},
\]
and define
\[
A = \operatorname{softmax}\!\Bigl(\frac{QK^\top}{\sqrt{n_{\mathrm{attn}}}}\Bigr)V_{\mathrm{val}},
\qquad
Y_{\mathrm{attn}} = A\,\wbox{W_O^\top}.
\]
We use \(V_{\mathrm{val}}\) here for the value vectors to avoid notational conflict with the weight-site matrix \(V\).

The MLP is then
\[
H_{\ell+1} = \sigma\!\bigl(H_\ell \wbox{W_\ell^\top} + b_\ell\bigr),
\qquad \ell=1,\dots,L-1,
\qquad H_1 := Y_{\mathrm{attn}},
\]
with final output
\[
Y = H_{L-1}\,\wbox{W_L^\top} + b_L.
\]

The relevant shapes are
\[
X \in \R^{B \times T \times n_{\mathrm{in}}},
\qquad
A \in \R^{B \times T \times n_{\mathrm{attn}}},
\qquad
H_\ell \in \R^{B \times T \times n_{\mathrm{mlp}}}
\quad \text{for } \ell=1,\dots,L-1.
\]

Concatenating the trainable weights into one block-diagonal matrix
\[
W = \mathrm{blockdiag}(W_Q,W_K,W_V,W_O,W_1,\dots,W_L),
\]
the corresponding weight-site matrix is
\[
V = \mathrm{cat}(X,X,X,A,H_1,\dots,H_{L-1}) \in \R^{B T \times m},
\]
where
\[
m = 3n_{\mathrm{in}} + n_{\mathrm{attn}} + (L-1)n_{\mathrm{mlp}}
\]
is the lifted feature dimension. Therefore
\[
\mathcal{K} = V V^\top \otimes I_n,
\]
for the appropriate output-side dimension \(n\) induced by the stacked weight parameterization. Expanding the Gram term gives
\[
\mathcal{K}
=
\bigl(3XX^\top + AA^\top + H_1H_1^\top + \cdots + H_{L-1}H_{L-1}^\top\bigr)\otimes I_n.
\]
Hence
\[
\rank(VV^\top) \le m = 3n_{\mathrm{in}} + n_{\mathrm{attn}} + (L-1)n_{\mathrm{mlp}}.
\]
In particular, the temporal structure of the global-state NTK is bottlenecked by the joint rank of the input, attention output, and hidden activations that immediately feed trainable weights.

\end{document}